\documentclass{article}
\usepackage{PRIMEarxiv}

\usepackage[utf8]{inputenc} 
\usepackage[T1]{fontenc}    
\usepackage{hyperref}       
\usepackage{url}            
\usepackage{booktabs}       
\usepackage{amsfonts}       
\usepackage{nicefrac}       
\usepackage{microtype}      
\usepackage{fancyhdr}       
\usepackage{graphicx}       

\usepackage{color}
\usepackage{mdframed} 
\usepackage{amssymb}   
\usepackage{titlesec}
\usepackage{glossaries}
\usepackage{subcaption}

\usepackage{tabularx}
\usepackage{times}
\usepackage{epsfig}
\usepackage{cite}

\usepackage{multirow}%
\usepackage{amsmath,amssymb}%
\usepackage{amsthm}%
\usepackage{mathrsfs}%
\usepackage[title]{appendix}%
\usepackage{xcolor}%
\usepackage{textcomp}%
\usepackage{manyfoot}%
\usepackage{booktabs}%
\usepackage{algorithm}%
\usepackage{algorithmicx}%
\usepackage{algpseudocode}%
\usepackage{listings}%

\newtheorem{theorem}{Theorem}[section]
\newtheorem{lemma}[theorem]{Lemma}%

\newtheorem{definition}{Definition}%
\newtheorem{corollary}{Corollary}%

\title{A Markovian View of Iterative-Feedback Loops in Image Generative Models: Neural Resonance and Model Collapse}

\author{Vibhas Kumar Vats$^1$,
David J. Crandall$^1$,
samuel Goree$^2$ \\
$^1$ Luddy School of Informatics, Computing, and Engineering, Indiana University, Bloomington, IN, USA\\
$^2$ Department of Computer Science, Stonehill College, Easton, MA, USA}

\begin{document}
\maketitle

\begin{abstract}
AI training datasets will inevitably contain 
   AI-generated examples, leading to ``feedback'' in which the output
   of one model impacts the training of another.  It is known that
   such iterative feedback can lead to model collapse, yet the
   mechanisms underlying this degeneration remain poorly understood.
Here we show that a broad class of feedback
processes converges to a low-dimensional
invariant structure in latent space, a phenomenon we call \emph{neural resonance.}
By modeling iterative feedback as a 
 Markov Chain, we 
 show that two conditions are needed for this resonance to occur:
\emph{ergodicity} of the feedback process and \emph{directional contraction} 
of the latent representation. 
By studying diffusion models on MNIST and ImageNet, as well as CycleGAN and  an audio feedback experiment, 
we map how local and global manifold geometry evolve, and we introduce an eight-pattern taxonomy of collapse behaviors. 
Neural resonance provides a
unified explanation for long-term degenerate behavior in generative models
and provides practical diagnostics
for identifying, characterizing, and eventually mitigating collapse.
\end{abstract}



\section{Introduction}
\label{sec:intro}

\begin{figure*}[t]
    \centering
    \includegraphics[width=\textwidth]{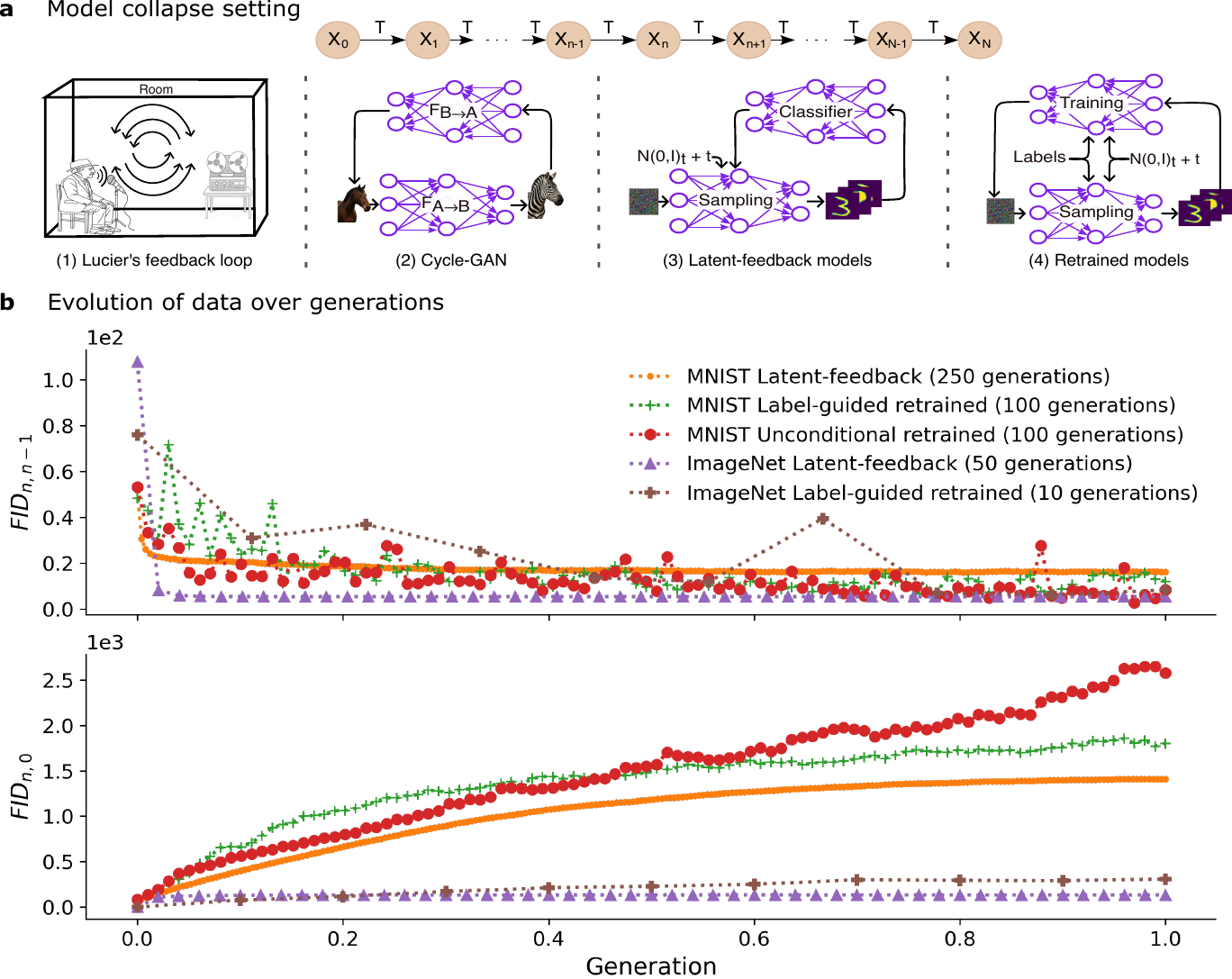}
    \caption{\textbf{A high-level representation of the Iterative feedback process.} (a) A graphical representation of 
    the generational Markov chain and cartoon overview of our iterative generation settings. $X_n$ represents the 
    current distribution of images at generation $n$.
    (b) Data distribution distances for two experiments in scenario 
    (3), and three experiments in scenario (4), measured using FID score. The top plot shows single-step changes
    between consecutive generations, while the bottom plot shows cumulative change from the original distribution. 
    Experimental runs with differing numbers of generations are scaled to the width of the graph, so that the x-axis represents fraction of total generations, and each point is one generation.}
    \label{fig:gen-chain-and-drivers}
\end{figure*}
%

Generative AI models are now so widely used that the text, imagery,
audio, and video that they create will inevitably ``pollute'' the
training datasets of the next generation of models --- and those
models, in turn, will pollute the training datasets of the subsequent
generation, and so on.  This can create an iterative feedback loop,
causing training sets to co-evolve with subsequent generations of
models~\cite{aastrom2021feedback} and drift away from the true data
distribution they were meant to represent.

Recent work has shown that
such feedback can eventually lead to ``model
collapse''~\cite{shumailov2024ai}, with potentially catastrophic
implications for modern AI systems.
However, the underlying dynamics of iterative feedback leading to model
collapse are poorly understood.  For example, do generative models
undergoing iterative feedback behave chaotically, or do they converge
to predictable stable points?  Do distributions of latent representations contract or
expand, do semantic modes survive or decay, and does generational
drift reach a stable point or does it continue forever?

Of course, systems undergoing iterative feedback are nothing new.  To
study latent feedback in AI models, we were inspired by a work of sonic art.  In 1969, Alvin Lucier's avant-garde composition \emph{I Am Sitting in a
Room}~\cite{lucier1969sitting} interrogates similar concepts. In it, Lucier speaks several sentences
into a tape recorder explaining his process, and then plays it back in his Connecticut apartment while
re-recording it on a second device. While the first recording
faithfully captures the words and his voice, after multiple iterations the sound of
his speech dissolves into tones that have more to do with
the geometry of the room and the noise characteristics of the
equipment than what he actually said. This degradation process can be
well explained by physical laws that govern how transformations affect audio
signals over time. We wondered: are there similar principles that explain
how AI models will react to iterative feedback?

In this paper we show, for the first time, that many generative
feedback processes follow a surprisingly predictable long-term
behavior.  Across several diverse model classes, including diffusion
models trained on their own outputs, deep style transfer models
applied repeatedly on the same image, and even Lucier's sonic art
piece, we find that iterative feedback often pushes the latent
representation toward a low-dimensional invariant structure. We call
this phenomenon \emph{neural resonance}, in analogy to resonance in
physical systems in which repeatedly-applied transformations can cause
some modes to be amplified while others are suppressed.

To characterize neural resonance and how it arises, we model iterative
feedback as a Markov Chain, choosing states
and transitions depending on the feedback situation to be modeled.
Using this general framework, we identify two conditions that are
needed for resonance to emerge: \emph{ergodicity}, meaning that the
chain converges to a unique stationary distribution that is
independent of the initial state~\cite{meynandtweedie2009}, and
\emph{directional contraction}, in which latent features shrink
towards a progressively smaller set of axes.  Under these conditions,
the system exhibits two distinct phases: an initial \emph{transient}
phase when the distribution changes rapidly from one iteration to the
next, followed by a \emph{stationary} phase in which the distribution
remains essentially unchanged.  When either condition fails,
the latent features expand and the system does not exhibit neural
resonance (Table \ref{tab:experiment-summary-sparkline}).

We introduce simple diagnostic measures for identifying resonance,
based on complementary measures of local and cumulative drift between
iterations.  We introduce a compact eight-pattern taxonomy of semantic
expansion and contraction behaviors
(Table~\ref{table:ideal-manifold-behaviors}) that describes how local
and global latent geometry evolve under feedback. Applying the
taxonomy across different types of feedback with different datasets
shows that data compressibility strongly influences the outcome:
highly compressible datasets such as MNIST keep their semantics over
generations, whereas more diverse datasets like ImageNet lose
semantics and collapse into very simple examples.

Taken together, our results provide a framework for understanding
the dynamics of generative feedback, including model collapse.  Neural
resonance explains why collapse occurs, how long-term latent
structures stabilize, and why models trained on diverse real data are
likely to outperform models trained on synthetic images.  Our findings
also suggest practical diagnostic tools for monitoring and mitigating
degeneration caused by iterative feedback, which may be essential in
future AI systems as synthetic data becomes increasingly common.

\section{Iterative Feedback as a Markov Chain}

We model the iterative feedback process as a Markov
Chain~\cite{meynandtweedie2009}.
We consider two main types of
iterative feedback. First, we consider feedback that is at the image
or sample level: a single image that is transformed back and forth
using CycleGAN, for example, or the audio signal that is iteratively
transformed in Lucier's piece.
In the first case, the state space of the Markov 
Chain consists of all possible images. Starting from an initial input image
$X_0$, each iteration of feedback is modeled as an operator
$T(\cdot)$ which produces $X_{n+1}=T(X_n)$ (see
Fig.~\ref{fig:gen-chain-and-drivers}a).
Second, we consider  
when the feedback is instead at the dataset level, such as a model that
is trained on its own previous outputs.
In this case, the state space conceptually
consists of all possible dataset distributions, and the system moves
from state to state to describe how distributions evolve across
successive generations.
Importantly, in either case, the process obeys the Markov property that 
future states (images or data distributions) depend only on the current state,
not the full history.

If a Markov Chain is ergodic, it converges to a unique 
stationary distribution, regardless of
the initial state (see \S3 of Supplementary for in-depth discussion).
In our context,
this corresponds to when the distribution of generated
images stabilizes such that images $X_{n+1}$ and $X_n$ are drawn from
statistically indistinguishable distributions.
In practice, this often means that the images generated by subsequent iterations
become very similar to one another.
To quantify convergence in Markov Chains, we monitor two complementary measures
of drift using Fréchet Inception Distances (FID): the local drift,
$\operatorname{FID}_{n,n-1}$, and the cumulative drift,
$\operatorname{FID}_{n,0}$, as shown in
Fig.\ref{fig:gen-chain-and-drivers}-b.  When both curves plateau, the
process is considered to have reached \emph{empirical stationarity},
meaning the change in FID score from one
generation to the next is less than some tolerance $\epsilon$.
Fig.\ref{fig:gen-chain-and-drivers}-b shows the local drift
$\operatorname{FID}_{n,n-1}$ (top) and the cumulative drift
$\operatorname{FID}_{n,0}$ (bottom) for three Markov Chains\textemdash
latent-feedback (conditioned on latent image representations), 
label-guided retrained (conditioned on class labels), and
unconditional-retrained (no conditionals)\textemdash evaluated on MNIST and ImageNet-5.
For MNIST, FID scores are computed in the latent space of a locally
trained classifier, while for ImageNet-5 they are computed in the
Inception-V3 feature space.  Across all settings, local drift
decreases rapidly and stabilizes, while cumulative drift rises and
then saturates for the latent-feedback and label-guided retrained
chains, yielding a joint plateau that signifies empirical
stationarity.  By contrast, the unconditional-retrained chain exhibits
persistent local and cumulative drift across all 100 generations,
indicating that it has not yet reached stationarity.  The
corresponding evolution of class exemplars is shown in
Fig. \ref{fig:result-examples}(a-e).

\section{Neural Resonance}

The convergence of a Markov Chain toward a stationary distribution mirrors the
acoustic phenomenon in Alvin Lucier's piece \emph{I Am Sitting in a
Room}.  In that work, each re-recording applies the room's linear
transfer operator: frequencies misaligned with the dominant eigenmodes
decay exponentially, leaving only the resonant `room chord.'  This
process illustrates a general rule: when a transformation is
repeated, the system selectively amplifies the modes it can sustain.
An analogous mechanism unfolds in neural network–based generative
feedback systems.  In an ergodic Markov Chain, which converges to a
unique stationary distribution independent of
initialization, repeated application of the feedback
operator $T$ filters the latent representation, damping off-manifold
directions and preserving modes aligned with the invariant manifold.
Ergodicity ensures that the Markov Chain eventually forgets its initialization
and converges toward a stationary distribution confined to a
contracted region of the latent manifold.  Contraction implies that
most representational axes in this space are progressively attenuated
under feedback, while a small subset persists across generations.
Iterative feedback thus concentrates representations onto this
invariant subset, analogous to physical resonance.
When this Lucier-like selective convergence appears within a
fixed latent feature space, we refer to it as \emph{neural resonance}.

Two conditions are jointly necessary to observe neural 
resonance: \emph{directional contraction} 
and \emph{ergodic Markov Chains}. 
Directional contraction shapes the geometric limit of the latent manifold by 
progressively shrinking and organizing its representation, 
thereby stabilizing the manifold’s effective degrees of freedom across generations.
As in Lucier’s piece, repeated playback preferentially reinforces the room’s 
resonant frequencies until they dominate across iterations. Similarly, in ergodic Markov Chains,
directional contraction selects latent modes while ergodicity sustains them, as exemplified by 
latent-feedback, label-guided retrained and unconditional retrained models.
By contrast, in non-ergodic Markov Chains, 
iterative feedback can select modes without sustaining them.
Deterministic or weak-noise feedback 
loops (e.g., CycleGAN) exemplify this behavior.

\begin{table}[t]
  \centering
  \caption{\textbf{Latent‑manifold dynamics under neural resonance.}
    Each row lists a distinct combination of directional change in the local spread ($\sigma_{intra}$), 
    the local curvature measure ($m_{LB}$), and the global participation ratio ($PR_G$). 
    Patterns are grouped into two regimes, semantic expansion ($\sigma_{intra}\uparrow$) 
    and semantic contraction ($\sigma_{intra}\downarrow$), and each is labeled 
    according to its geometric character (e.g., coherent expansion, wrinkled contraction). 
    An animation illustrating these patterns is provided in the Supplementary Video.}
  \begin{tabular}{llcccl}
    \toprule
    \multicolumn{2}{c}{\textbf{Manifold Behavior}} 
      & \multicolumn{2}{c}{\textbf{Local Metrics}} & \textbf{Global Metric} & \textbf{Dimensional Pattern}\\
    \cmidrule(lr){1-2} \cmidrule(lr){3-4} \cmidrule(lr){5-5}
    Local & Global & $\sigma_{intra}$ & $m_{LB}$ & $PR_G$ & \\
    \midrule
    \addlinespace[3pt]
    \multicolumn{6}{l}{\textbf{Semantic Expansion Regime}} \\
    \addlinespace[3pt]
    Expansion       & Expansion    & $\uparrow$ & $\uparrow$ & $\uparrow$ & Coherent Expansion (CE)\\[2pt]
    Expansion       & Contraction  & $\uparrow$ & $\uparrow$ & $\downarrow$ & Wrinkled Expansion (WE)\\[2pt]
    Expansion     & Expansion    & $\uparrow$ & $\downarrow$ & $\uparrow$ & Anisotropic Expansion (AE)\\[2pt]
    Expansion     & Contraction  & $\uparrow$ & $\downarrow$ & $\downarrow$ & Oblate Expansion (OE)\\[2pt]
    \midrule
    \addlinespace[2pt]
    \multicolumn{6}{l}{\textbf{Semantic Contraction Regime}} \\
    \addlinespace[2pt]
    Contraction    & Contraction  & $\downarrow$ & $\downarrow$ & $\downarrow$ & Coherent Contraction (CC)\\[2pt]
    Contraction   & Expansion    & $\downarrow$ & $\downarrow$ & $\uparrow$ & Anisotropic Contraction (AC)\\[2pt]
    Contraction     & Contraction  & $\downarrow$ & $\uparrow$   & $\downarrow$ & Wrinkled Contraction (WC)\\[2pt]
    Contraction     & Expansion    & $\downarrow$ & $\uparrow$   & $\uparrow$ & Oblate Contraction (OC)\\[2pt]
    \bottomrule
  \end{tabular}
  \label{table:ideal-manifold-behaviors}
\end{table}
\subsection{Latent Manifold Dynamics}

We characterize latent-manifold dynamics in terms of dimensionality
using two complementary metrics: the participation ratio, $PR_G$, which tracks
global dimensionality; and the Levina-Bickel intrinsic dimension\cite{Levina2005}, $m_{LB}$, which tracks 
local dimensionality. We also use intra-class spread, $\sigma_{intra}$, measuring classwise dispersion,
to observe local expansion and contraction. 
Under iterative feedback, components orthogonal to the eventual invariant subspace are progressively damped, 
and thus $PR_G$ contracts and stabilizes 
at a floor reflecting the dimension of that subspace. Within that emerging subspace, 
local structure can still reconfigure\textemdash both $m_{LB}$ 
and $\sigma_{intra}$ may transiently 
rise or fall before settling\textemdash without undoing the global contraction.

We characterize changes in the latent manifold by examining 
metric trajectories across generations. 
Taken together, $PR_G$ (global) and $m_{LB}$ (local) delineate eight dimensional 
patterns that summarize how local and global manifold geometry co-evolve 
under feedback (see Table~\ref{table:ideal-manifold-behaviors}); 
$\sigma_{intra}$ indicates local expansion versus contraction regimes. These patterns form pairs, where one pattern is the opposite of another.
These patterns are not always intuitive; physical analogies, 
together with Supplementary video animation, clarify them.

\begin{itemize}
  \item \emph{Coherent expansion (CE).} $\sigma_{intra}\!\uparrow$ activates new axes locally and globally: $PR_G\!\uparrow$, $m_{LB}\!\uparrow$. 
The manifold simultaneously puffs out and unfolds, like an inflating balloon.  
\textit{Coherent contraction (CC)} is the reverse.
  \item \emph{Oblate expansion (OE).} $\sigma_{intra}\!\uparrow$ along one or a few dominant 
directions while collapsing elsewhere: $PR_G\!\downarrow$, $m_{LB}\!\downarrow$. 
The manifold flattens and spreads along a few axes, like a balloon squeezed between hands, 
expanded sideways along a few axes while collapsed along others.  
\textit{Oblate contraction (OC)} is the reverse.
  \item \emph{Wrinkled expansion (WE).} $\sigma_{intra}\!\uparrow$ forms small-scale 
wrinkles or folds: $m_{LB}\!\uparrow$ while $PR_G\!\downarrow$. 
The manifold fans out locally as fine wrinkles form, reducing global spread, like crumpled fabric 
with wrinkles.  
\textit{Anisotropic contraction (AC)} is the reverse.
  \item \emph{Anisotropic expansion (AE).} $\sigma_{intra}\!\uparrow$ stretches clusters along different axes: $m_{LB}\!\downarrow$, $PR_G\!\uparrow$. 
Local structure collapses along stretched directions while new global axes are activated,
like multiple pieces of play-dough being stretched in different directions, each elongating along 
its own axis while globally occupying a larger space.  
\textit{Wrinkled contraction (WC)} is the reverse.
\end{itemize}

\begin{table}[t]
\caption{Summary of observed latent manifold dynamics with sparkline plots. The 
dynamics is indicated with Coherent Expansion (CE), Wrinkled Expansion (WE),
Oblate Expansion (OE), Coherent Contraction (CC), and Wrinkled Contraction (WC). LGR and LF 
means label-guided retrained and latent-feedback models, respectively. CyclaGAN experiment uses 
Horse-Zebra data. See Supplementary Materials for full plots with axes.}
\centering
\begin{tabular}{l@{\hskip 2pt}c@{}c@{}c@{}c}
\toprule
\textbf{Experiments} & $\mathbf{\sigma_{intra}}$ & $\mathbf{m_{LB}}$ & $\mathbf{PR_G}$ & \bf{Dynamics}\\
\midrule
LGR-MNIST              
& \includegraphics[width=3cm]{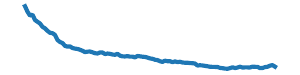}
& \includegraphics[width=3cm]{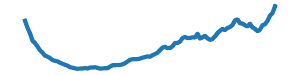} 
& \includegraphics[width=3cm]{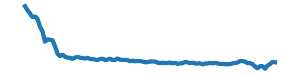}
& CC$\rightarrow$WC
\\[3pt]
LF-MNIST         
& \includegraphics[width=3cm]{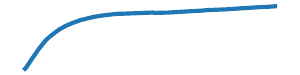}
& \includegraphics[width=3cm]{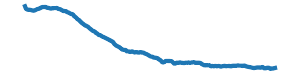} 
& \includegraphics[width=3cm]{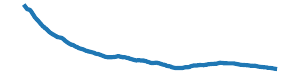}
& OE
\\[3pt]
LGR-ImageNet5       
& \includegraphics[width=3cm]{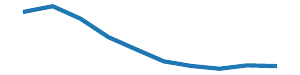}
& \includegraphics[width=3cm]{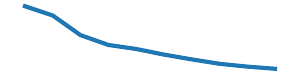} 
& \includegraphics[width=3cm]{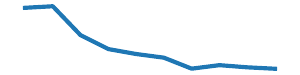}
& CC
\\[3pt]
LF-ImageNet5       
& \includegraphics[width=3cm]{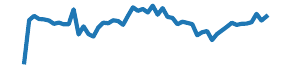}
& \includegraphics[width=3cm]{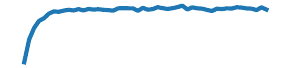} 
& \includegraphics[width=3cm]{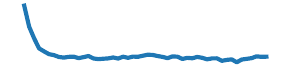}
& WE
\\[3pt]
CycleGAN  
& \includegraphics[width=3cm]{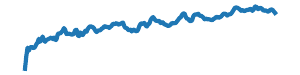}
& \includegraphics[width=3cm]{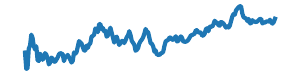} 
& \includegraphics[width=3cm]{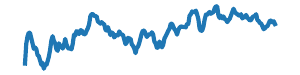}
& CE\\
\bottomrule
\end{tabular}
\label{tab:experiment-summary-sparkline}
\end{table}

Empirically, we observe five of the eight possible dimensional patterns across our 
iteration experiments, see Table \ref{tab:experiment-summary-sparkline}. 
These patterns are not tied to a single feedback regime: 
different scenarios can express the same pattern, 
and a given scenario can transition between patterns over time. 
First, distinct Markov Chains can show the same local–global geometry 
with opposite locality: both LF-MNIST and LGR-ImageNet-5 exhibit 
decreasing $PR_G$ and $m_{LB}$, but $\sigma_{intra}$ rises in LF-MNIST 
(oblate expansion, OE) and falls in LGR-ImageNet-5 (coherent contraction, CC), 
revealing expansion vs. contraction regimes. 
Second, similar Markov Chains can diverge by dataset: LF-MNIST follows OE, 
whereas LF-ImageNet-5 follows wrinkled expansion (WE). 
Third, a single chain can have multiple patterns over generations: 
LGR-MNIST begins in CC and later shifts to wrinkled contraction (WC; $m_{LB}$ turns upward). 
By contrast, CycleGAN is non-ergodic: trajectories do not settle into a single invariant subspace 
but switch among multiple modes (see Fig.\ref{fig:result-examples}-f), with significant step-to-step fluctuations in $PR_G$, $m_{LB}$, and $\sigma_{intra}$; 
the overall trend follows coherent expansion (CE), but it does not exhibit neural resonance.

\begin{figure*}[t]
    \centering
    \includegraphics[width=1\textwidth]{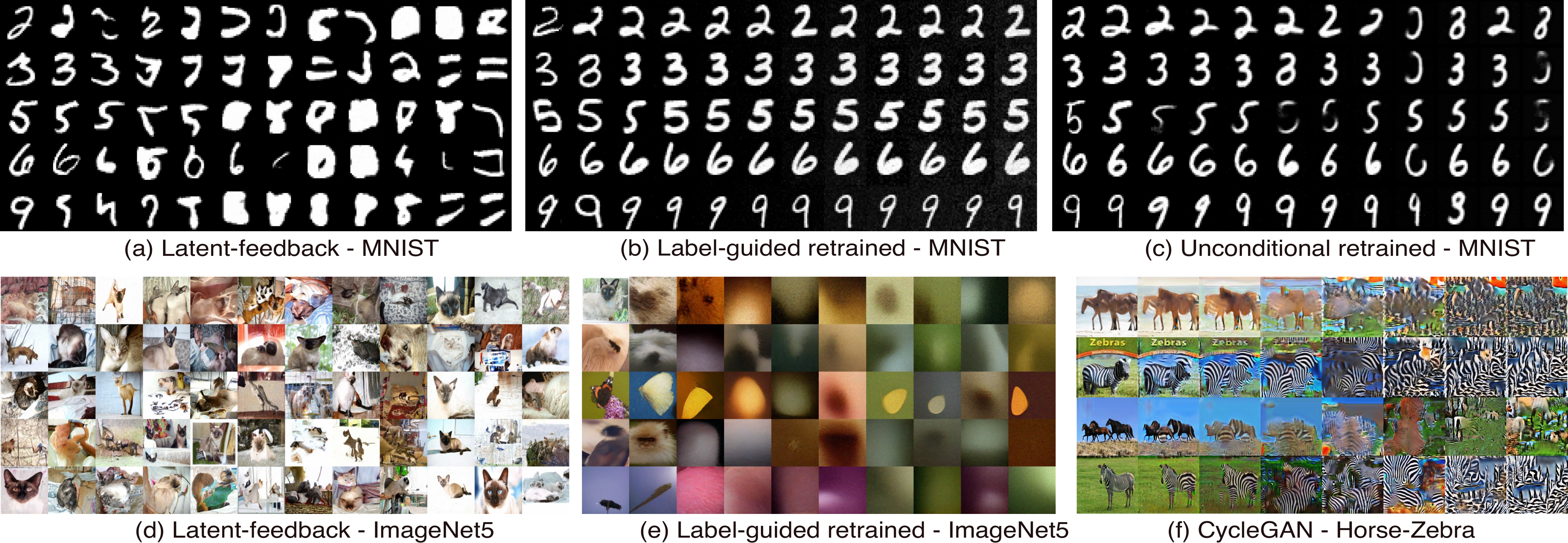}
    \caption{\textbf{Evolution of class exemplars under iterative feedback.} 
    Columns: same-generation samples.
    (a–c) MNIST: latent-feedback progressively loses semantic coherence; 
    label-guided and unconditional retraining retain class identity 
    but converge to repetitive templates owing to high compressibility. 
    (d–e) ImageNet-5: both regimes collapse, but differently: latent-feedback preserves 
    coarse class cues while producing entangled objects; label-guided retraining loses most class semantics.
    (f) CycleGAN (non-ergodic): trajectories settle into distinct attractor
    basins between the two domains. Random examples from the 
    same class are shown for each scenario.}
    \label{fig:result-examples}
\end{figure*}

\subsection{Data Compressibility}

We find that generational dynamics are sensitive to data compressibility. 
On MNIST, a highly compressible dataset, label-guided and unconditional retraining 
quickly contract the latent manifold yet retain task-relevant semantics, see Fig.~\ref{fig:result-examples}(b–c). 
The outcome is increasing repetitiveness rather than outright semantic failure. 
In the latent-feedback setting on MNIST (Fig.~\ref{fig:result-examples}-a), semantics persist even longer: 
the chain drifts slowly toward collapse, with 
fine details gradually eroding while digits remain 
recognizable over many generations. 
By contrast, on a diversity-demanding dataset (ImageNet-5), 
semantic coherence collapses rapidly: within five feedback rounds 
the samples lose most object semantics, and the chain converges to 
low-entropy textures or generic color blobs despite identical 
sampling budgets and optimization settings (Fig.~\ref{fig:result-examples}-e).

Taken together, these results demonstrate that collapse is not an artifact of a particular feedback 
regime: across regimes, compressibility primarily governs the timeline 
and the failure mode: highly compressible data drift into repetition, 
whereas diversity-demanding data undergo early semantic erosion.

\section{Discussion}

We have shown that iterative generative feedback creates a unifying behavior, \emph{neural resonance},
in which representations converge to a low-dimensional invariant manifold when (and only when) directional contraction 
and ergodicity are both present. 
The resulting dynamics exhibit a transient–to–stationary progression 
that we diagnose in practice with complementary drift measures (local and cumulative FID). 
To interpret how manifold geometry evolve, we introduced a compact eight-pattern taxonomy linking 
local and global dimensional behavior; across regimes we observe five of these patterns. 
We also discuss non-ergodic feedback systems that switch among multiple modes rather than settling into a single invariant subspace, 
and therefore do not exhibit neural resonance. 
Data compressibility emerges as an important driver of outcomes: highly compressible data drift toward repetition 
while retaining semantics longer, whereas diversity-demanding data undergo early semantic erosion. 
Taken together, these risks create a practical asymmetry: models trained early on cleaner corpora/datasets 
enjoy a first–mover advantage, whereas downstream learners trained on synthetic–heavy 
mixtures face distribution shift, loss of rare concepts, and accelerated degeneration.


\section{Methods}\label{sec:methods}

We study five iterative-feedback scenarios: 
(i) a functional analogue of Lucier’s feedback loop, 
implemented by repeatedly filtering audio through 
measured impulse responses of physical spaces~\cite{abel2012luciverb}; 
(ii) CycleGAN, where pre-trained image-to-image translators 
are reapplied in a loop across domains~\cite{zhu2017unpaired}; 
and three diffusion-based settings on MNIST and ImageNet-5\textemdash 
(iii) latent-feedback, where diffusion model is conditioned on 
latent image representations, (iv) label-guided retrained, where diffusion model is 
conditioned on class-labels, 
and (v) unconditional retrained models, where no conditional information is used. 
In all cases, we view one application of the operator as a `generation'
and compose it over $n{=}0,\dots,N$.

First two experiments, Lucier's functional analogue and CycleGAN serve as non-ergodic chains, 
and the rest three experiments are ergodic chains. 
The \emph{unconditional retrained} model uses no 
conditioning and is trained \emph{from scratch} 
at each generation on samples produced by the previous generation. 
The \emph{label-guided retrained} model conditions 
on class labels via the standard conditional-diffusion 
interface and is likewise retrained from scratch 
each generation on the newly generated dataset. 
The \emph{latent-feedback} model is trained \emph{once} 
on the original dataset and then held fixed; 
conditioning is provided by $n$-dimensional 
feature vectors extracted from a classifier 
(locally trained for MNIST; Inception-V3 for ImageNet-5). 
Only the conditioning inputs—i.e., the images used 
to compute features—participate in the feedback loop, 
while the diffusion backbone remains unchanged. 
This design isolates the role of conditioning: 
retrained models evolve parameters across generations, 
whereas latent-feedback applies a fixed generator 
driven by iteratively updated conditioning signals.
All diffusion-based experiments operate 
under matched sample budgets and 
sampler schedules.

\subsection{Lucier's Feedback Experiment}

We do not physically reproduce Lucier’s sonic art piece; 
instead, following Abel et al.~\cite{abel2012luciverb}, 
we implement a \emph{functional analogue} of the process 
as a time–homogeneous Markov operator. 
Specifically, let $h$ denote an acoustic 
impulse response (IR) of a physical space and $x$ an audio signal. 
One generation applies the linear time–invariant mapping
\begin{equation}
    X_{n+1}\;=\;T_h(X_n)\;=\;\mathrm{norm}\!\bigl(X_n * h\bigr),
\end{equation}
where $*$ denotes convolution and $\mathrm{norm}(\cdot)$ rescales the output 
to a fixed reference level to avoid trivial decay or blow-up. 
This induces a deterministic, time-homogeneous kernel $K(X,\cdot)=\delta_{T_h(X)}$; 
because no stochasticity is injected, the chain lacks 
one-step positive density and serves as a non-ergodic comparator in our study.

To instantiate $T_h$, we selected IRs from \emph{ten} 
distinct physical spaces in the OpenAIR corpus~\cite{carslaw2012openair} 
and iteratively filtered a continuous reading of \emph{War and Peace} (Leo Tolstoy) 
for $N{=}50$ generations per IR. 
Audio was converted to mono, amplitude-normalized, 
and convolved with each IR (FFT-based implementation); 
outputs were re-normalized at every generation and 
fed back as inputs to the next. 
We emphasize that this is \emph{not} Lucier’s room–record/replay loop, 
but a fixed-operator analogue that captures the same feedback motif. 
Consistent with its non-ergodic nature, 
this system does not exhibit neural resonance in our sense: 
trajectories approach absorptive attractors determined by 
the filter, rather than a unique, noise-smoothed stationary distribution.

\subsection{CycleGAN Experiment}

For image–to–image translation we use a pre-trained 
CycleGAN~\cite{zhu2017unpaired} as operator $T$ in an 
iterative loop. CycleGAN comprises two generators, $F_{A\to B}$ 
and $F_{B\to A}$, that map between domains $A$ and $B$ 
(here, \emph{Horse} and \emph{Zebra}). 
Starting from each domain separately, we define a time-homogeneous update by composing the translators:
\begin{equation}
    X^{(A)}_{n+1}\;=\; F_{B\to A}\!\big(F_{A\to B}(X^{(A)}_{n})\big), 
\qquad
X^{(B)}_{n+1}\;=\; F_{A\to B}\!\big(F_{B\to A}(X^{(B)}_{n})\big),
\end{equation}
and iterate for $N{=}200$ generations in each direction 
(Horse$\to$Zebra$\to$Horse and Zebra$\to$Horse$\to$Zebra). 
Inputs are drawn from the standard Horse$\leftrightarrow$Zebra dataset; 
images are resized to the model’s native resolution 
and intensity-scaled to the network’s input range before each application. 
With weights fixed and no stochasticity injected, this induces a deterministic, 
time-homogeneous Markov kernel that 
lacks one-step positive density and is therefore non-ergodic.

Empirically, the loop does not exhibit neural resonance: 
trajectories tend to settle into attractor basins (texture/stripe templates) 
that depend on the starting image and domain, yielding apparent drift 
plateaus that reflect absorption (as described in Shumailov et al.~\cite{shumailov2024ai}) rather than convergence to a unique, 
noise-smoothed stationary distribution. 
Full manifold-geometry curves are reported in the Supplementary Materials.

\subsection{MNIST Experiments}

All MNIST experiments (latent–feedback, label–guided retrained, unconditional retrained) 
use the \emph{same} diffusion backbone and training recipe, 
following Dhariwal and Nichol’s U-Net design \cite{dhariwal2021diffusion}. 
We employ a 64–channel U-Net with self–attention at $14\times14$ and $7\times7$ spatial resolutions,
trained for 200,000 optimizer steps with batch size 256 using Adam. 
Unless otherwise noted, we use 1,000 diffusion steps with a linear noise schedule 
and the same sampler schedule across scenarios. The image resolution is 
MNIST native ($28\times 28$, single channel); 
all other hyperparameters are matched across scenarios (full details in Code).

\noindent \textbf{Retrained diffusion, label–guided and unconditional: }
For the \emph{unconditional retrained} model, no conditioning is used. 
For the \emph{label–guided retrained} model, class labels provide 
conditioning via the standard conditional–diffusion interface. 
In both cases, 
at generation $n$ we train a fresh model $\theta_{n+1}$ from scratch solely 
on the synthetic dataset produced by the previous generation, $D_n$, for a total of $N{=}100$ iterations; 
no real data are mixed in subsequent generations. 
Each generation we sample 5,000 images per class (50,000 total) 
to form $D_{n+1}$, which then becomes the training set for subsequent generation. 
Sampler steps, schedules, and seeds are held fixed across generations.

\noindent \textbf{Latent–feedback diffusion: }
For the \emph{latent–feedback} model, 
a single conditional diffusion model $\theta_0$ is trained once on the MNIST training set 
and then \emph{frozen}. Conditioning signals are $n$–dimensional 
feature vectors extracted from a pre–trained MNIST classifier, 
specifically from a high-level convolutional layer before global pooling. 
Let $f(\cdot)$ denote this feature map. At generation $n$, 
given conditioning images $\{x^{(i)}_n\}$, we compute $c^{(i)}_n{=}f(x^{(i)}_n)$ 
and generate $x^{(i)}_{n+1}\sim p_\theta\big(\,\cdot \mid c^{(i)}_n\big)$,
producing one image per conditioning vector. 
The first generation uses the 10,000 MNIST \emph{test} images 
as conditioning inputs (yielding 10,000 generated images); 
thereafter, only the \emph{conditioning images} participate in the feedback loop for next 250 iterations: 
the synthetic set from generation $n{+}1$ becomes the conditioning set for generation $n{+}2$, 
while the diffusion backbone and all sampling hyperparameters remain unchanged. 
This design isolates feedback to the conditioning stream, 
in contrast to the parameter–evolving retraining setups. 
Additional training and sampling details, including exact 
feature dimensionality and classifier architecture, are provided with the code.

\subsection{ImageNet-5 Experiments}

Both ImageNet-5 experiments\textemdash latent–feedback and 
label–guided retrained diffusion\textemdash use the \emph{same} backbone 
and training recipe, following the U-Net design of Dhariwal and Nichol~\cite{dhariwal2021diffusion}.
We employ a 128–channel U-Net with self–attention at 
spatial resolutions $32\times32$, $16\times16$, 
and $8\times8$, trained for 300,000 optimizer steps with batch 
size 40 using Adam. Unless noted, we use 1,000 diffusion steps 
with a linear noise schedule and identical sampler settings across scenarios. 
Images are processed at $128\times128\times3$ with standard normalization. 
Hyperparameters, augmentation, and optimizer details are 
matched across scenarios.

\noindent \textbf{Label–guided retrained diffusion: }
The \emph{label–guided retrained} model conditions on class 
labels via the standard conditional–diffusion interface and 
follows an \emph{evolving–operator} protocol: at generation $n$, 
a fresh model $\theta_{n+1}$ is trained \emph{from scratch} 
solely on the synthetic dataset produced by the previous generation, $D_n$. 
Each generation we sample 50,000 images (5,000 per class) to form $D_{n+1}$, 
which becomes the training set for subsequent generation ($N{=}10$). 
No real images are mixed after $n{=}0$. 
Sampler schedules, diffusion steps, and seeds are held fixed across generations.

\noindent \textbf{Latent–feedback diffusion: }
The \emph{latent–feedback} model is trained \emph{once} on 50,000 ImageNet-5 training images 
and then frozen; feedback acts only through the conditioning stream. 
Conditioning vectors are extracted with a pre–trained Inception–V3: 
for an image $x$, we compute a high–level feature $c{=}f(x)$. 
For sampling at generation $n{=}0$, we randomly select 10,000 images (2,000 per class) 
to form the conditioning set $\{c^{(i)}_0 {=} f(x^i)\}$ and generate one sample per conditioning vector.
Thereafter, the \emph{synthetic} images from generation $n$ are fed back to 
the feature extractor to produce conditioning vector for generation $n{+}1$ and so on for $N{=}50$, 
while the diffusion backbone $\theta_0$ and all sampling hyperparameters remain unchanged.

\newpage
\appendix

\section{Related Work}\label{sec:related-works}

Recent investigations into iterative feedback loops focus primarily on two 
domains: \emph{language models} 
and \emph{image-generative models}. In both areas, researchers explore two 
training protocols\textemdash fine-tuning and full-retraining. 
\emph{Fine-tuning} studies start  
with a pretrained network and repeatedly train on its own synthetic outputs,
typically uncovering semantic drift or diversity loss. 
By contrast, \emph{full-retraining} approaches initialize fresh models 
from scratch at every generation and optimize them using previous generation's data.
In this section, we revisit both model types.

\subsection{Language Models}

Shumailov et al.~\cite{shumailov2024ai} find model collapse in fine-tuned large-scale language models (LLMs)
in a fully-synthetic iterative feedback loop. An iterative feedback loop arises when successive 
generations of generative models are trained (or fine-tuned) on data 
produced by their predecessors. It is called `fully-synthetic' when the model is trained only with 
the generated data. They formally define collapse as the progressive loss of 
low-probability (tail) events, ultimately converging to a narrow, 
low-variance distribution. The paper traces the root causes 
to (i) \emph{statistical-approximation error}\textemdash finite sampling inevitably drops rare events,
(ii) \emph{functional-approximation error} from limitations of the learning procedure, and 
(iii) \emph{functional expressivity error} from imperfect model capacity.
Together these errors accumulate, causing the distribution to drift progressively
away from the true one.
The authors conclude that even tiny proportions of synthetic data, 
if recursively reused without continual infusion of fresh real samples, 
leads to irreversible forgetting. The authors call for watermarking, 
provenance tracking, and systematic infusion of real data to preserve model quality.

Extending these findings, Dohmatob et al.~\cite{Dohmatob2024scaling} re-examine 
this phenomenon in a simplified setting of kernel regression. 
They show how adding synthetic data changes scaling laws and they propose adaptive regularization 
that partially  mitigates collapse.
They provide a theoretical framework of model collapse through the 
lens of scaling laws and propose an adaptive regularization technique to 
at least partially mitigate model collapse.

Guo et al.~\cite{guo2024curious} probe the language richness 
of LLMs under iterative fine-tuning on their own outputs.
They introduce a three-tier evaluation suite that measures 
lexical, syntactic, and semantic diversities. 
Across all metrics, linguistic variety declines monotonically 
across fine-tuning cycles: vocabulary shrinks, 
syntactic constructions homogenize, and Sentence-BERT 
embeddings cluster more tightly.
They conclude that synthetic-only fine-tuning poses 
a systemic risk of linguistic flattening even 
when overall performance looks healthy, and they recommend diversity-aware objectives 
and routine mixing of fresh human data to avoid long-term erosion of language richness. 

Briesch et al.~\cite{briesch2024largelanguagemodelssuffer} present a
study of iterative feedback loops in language models trained from scratch.
They experiment with a logic-expression dataset\textemdash where correctness can be unambiguously evaluated\textemdash 
to measure both accuracy and diversity across generations. 
They find that while syntactic and semantic correctness of generated 
expressions initially improves over early generations, 
lexical diversity monotonically declines, 
collapsing toward nearly zero diversity within 50 full-synthetic generations. 
They further show that incremental or expanding data cycles\textemdash where a fraction of 
real data is retained or added each round\textemdash can slow but not prevent diversity loss.

\subsection{Image Generative Models}

Model collapse in image synthesis models was  
investigated by Shumailov et al.~\cite{shumailov2024ai}, 
who demonstrated\textemdash with a VAE trained on MNIST\textemdash that repeatedly fine-tuning 
a model on its own samples drives the latent manifold toward degeneration.  
In their experiment, each generation draws fresh latent codes from a Gaussian prior, 
decodes them into images, and then treats these outputs as the sole training 
data for the next iteration. After only a handful of iterations, 
the VAE’s representations collapse toward the latent mean: 
classes blur and become indistinguishable, and the encoder assigns high 
confidence to increasingly featureless digits\textemdash marking a collapse
of both diversity and semantics. 
Since this seminal result, researchers have
studied model collapse under different architectures and training 
schemes \cite{alemohammad2024selfconsuming,gerstgrasser2024is,Hataya2023ICCV,bohacek2023nepotistically,bertrand2024stability,yoon2024model,gillman2024selfcorrecting},
while theoretical work probes its statistical and dynamical underpinnings, 
proposing formal criteria and mitigation strategies \cite{shi2025closer,taori2023datafeedback,fu2024theoretical}.

Alemohammad et al.~\cite{alemohammad2024selfconsuming} coin the term \emph{Model Autophagy Disorder} (MAD) 
to describe this collapse by analogy to mad cow disease.
They analyze three feedback regimes: (i) a \emph{fully-synthetic} loop, 
where each new model is trained only on data from its predecessor; 
(ii) a \emph{synthetic-augmentation} loop, which retains a fixed cache 
of real images in each round; and (iii) a \emph{fresh-data} loop, which injects 
new real samples in every generation. Across multiple models 
and different datasets, 
they show that without sufficient fresh data injection in every generation, 
visual quality (precision) or diversity (recall) collapse within 
only a few iterations. They conclude that any self-consuming 
pipeline lacking fresh data ultimately collapses, 
progressively amplifying artifacts or shrinking support.
Gerstgrasser et al.~\cite{gerstgrasser2024is} arrive at 
a similar conclusion from an error analysis perspective. 
They find that replacing the real corpus with model outputs 
drives test loss upward. By contrast, concatenating each 
synthetic batch with the existing real-plus-synthetic pool 
keeps error stable and preserves both quality and diversity, 
without requiring special regularizers or temperature-sampling tricks.
Bertrand et al.~\cite{bertrand2024stability} introduce a fine-tuning framework 
based on mixed dataset\textemdash composed of real and 
synthetic images\textemdash 
and show that the long-term fate of an iterative feedback loop depends on 
one scalar: the per-round real-to-synthetic ratio $\lambda$. They also highlight that 
the critical $\lambda$-threshold is architecture- and resolution-dependent. 
Together, they establish that 
a competent seed model along with enough real data per round can stabilize model collapse, 
but removing real images drives the process toward degeneration.

Hataya et al.~\cite{Hataya2023ICCV} use large scale datasets like ImageNet and COCO  with diffusion
models to study the impact of generation collapse on future datasets and model performance. 
They replace varying proportions
of real ImageNet or COCO images with generated fakes 
to mimic future web contamination. On the ImageNet classification task, ResNet-50, Swin-S and ConvNeXt-T 
suffer only minor top-1 drops up to 40\% synthetic data, 
but accuracy plummets by 10-15 percentage points at 80\%  
and collapses to 15\%  when training solely on synthetic images. Experiments on a COCO captioning task show the same 
monotonic degradation. For image generation, an ID-DPM~\cite{dhariwal2021diffusion} trained on the mixes exhibits 
rising precision but sharply falling recall, revealing mode contraction. 
Similarly, Yoon et al.~\cite{yoon2024model} also explore model collapse 
in text-to-image generative models, where a pretrained text-to-image diffusion model 
is repeatedly finetuned on its own synthetic outputs (without the addition of real data) 
using a fixed prompt set. They first conduct extensive empirical 
studies across four diverse datasets 
(Pokémon, Kumapi, Butterfly, CelebA-1k), 
varying common hyperparameters to find that 
classifier–free guidance (CFG) scale is by far the dominant factor 
for model collapse. Low CFG (e.g. 1.0) yields smooth 
but blurry low-frequency degradation, high CFG (e.g. 7.5) produces 
saturated, repetitive high-frequency artifacts, 
and an intermediate CFG (dataset-specific, around 2.5 or 1.5) best 
balances fidelity and stability.
Inspired by genetic mutations, they propose ReDiFine (reusable diffusion fine-tuning), 
which combines condition-drop fine-tuning (to inject randomness) with CFG 
scheduling as a plug-and-play remedy. It dramatically reduces collapse 
rate while preserving first-iteration quality without data specific hyperparameter tuning.

Bohacek and Farid~\cite{bohacek2023nepotistically} demonstrate that 
Stable Diffusion v2.1 is highly susceptible to 
iterative feedback: after just five iterations of retraining on its own outputs,
the generated images of human faces
become severely distorted and homogeneous, 
even when the retraining set contains as few
as 3\% synthetic images mixed with 97\% real data. 
Across mixtures from 3 to 100\% synthetic images, FID rises and CLIP 
similarity falls sharply after a brief initial uptick, 
confirming rapid declines in both photorealism and semantic alignment. 
They also perform a healing experiment\textemdash five additional 
epochs on only real images\textemdash and find that it somewhat restores mean FID 
and most of the CLIP score, yet residual artifacts and elevated 
CLIP variance indicate that the damage is only partly reversible.  
They conclude that even minimal contamination of future 
datasets with model-generated content can trigger irreversible 
quality and diversity loss.
Gillman et al.~\cite{gillman2024selfcorrecting} address the instability and collapse by
proposing a self-correction framework designed to stabilize this process. They implement scalable 
self-correction via expert knowledge\textemdash K-means `anchor' projections 
for correcting MNIST-based models 
and a physics-simulator-based corrector for human motion models\textemdash and demonstrate 
on a challenging text-to-motion diffusion task that self-corrected loops 
avoid collapse even at 100\% synthetic data, 
consistently outperforming uncorrected loops.
Their results suggest 
that a reliable correction function 
can safeguard generative training against the degenerative effects 
of synthetic-only retraining.

Our work pushes the study of model collapse beyond 
conventional fine-tuning and full-retraining loops by dissecting 
three complementary feedback regimes through the lens of Markov processes.  
(i) In latent-feedback models, 
a frozen, pretrained classifier supplies $n$-dimensional  
feature vectors that serve as high-level conditioning for 
the generative model at every generation.  
(ii) Label-guided retraining forgoes pretrained features: 
each new generator is trained from scratch, conditioned only on class labels, 
allowing us to isolate the effect of coarse-grained supervision.  
(iii) Unconditional retraining removes all guidance, exposing the dynamics 
of a purely self-referential pipeline. By contrasting these settings we reveal 
how the type and granularity of conditioning govern the speed, form, 
and the inevitability of collapse under fully synthetic-data conditions.
Finally, we interpret collapse through the notion of \emph{neural resonance}, 
defined as the emergence of a low-dimensional invariant subspace that remains 
unchanged under repeated generative updates; our analysis shows 
that each feedback regime converges toward such a resonant subspace 
at a characteristic rate, providing a unifying explanation 
for the observed patterns of degradation.

\section{Datasets \texorpdfstring{$\&$}{and} Metrics}\label{sec:dataset-and-metrics}

This section details the datasets employed in our experiments and outlines the evaluation metrics 
used to probe the Markov chains. 

\subsection{Datasets}\label{subsec:dataset}

We evaluate our framework on three datasets: MNIST 
(Modified National Institute of Standards and Technology)~\cite{lecun2010mnist},
ImageNet-5, and OpenAIR~\cite{carslaw2012openair}.

The MNIST  dataset is a standard benchmark of 70,000 
grayscale 28$\times$28 images of handwritten digits (0-9). 
The dataset is divided into a training set of 60,000 
and a test set of 10,000 images.

ImageNet-5 is a custom, five-class subset of ImageNet \cite{russakovsky2015imagenet},  
created by leveraging its hierarchical label structure.
We merged related ImageNet synsets into five 
coarse categories\textemdash Cat, Terrier Dog, Insects, Monkey, and Wading-bird, 
each containing $\sim$10,000 color (RGB) images of resolution $128 \times 128$ pixels. 
The exact synset mappings needed to reproduce this subset are included with our code.

To emulate Lucier’s feedback loop, we convolved a continuous narration of \emph{War \& Peace} \cite{tolstoy2018war} 
with room-impulse responses from ten acoustically distinct spaces selected from the OpenAIR corpus \cite{carslaw2012openair}. 
Each impulse response defines a distinct class, yielding ten sound categories.
The full list of rooms and their corresponding impulse-response files are distributed with our code.

\subsection{Quantitative Metrics}\label{subsec:metrics}

We quantify the empirical behavior of each Markov chain using the \emph{Fréchet Inception Distance} (FID). 
The distance from one iteration (or generation) to the next measures \emph{local drift}, whereas distance 
from the original generation to a later generation measures \emph{cumulative drift} from the data manifold; 
the joint trajectory of these two curves provide evidence 
of empirical stationarity and the onset of neural resonance.
To probe the accompanying changes in latent geometry we employ three complementary 
statistics: (i) the within-class spread $\sigma_{intra}$, which measures class-specific spread; 
(ii) the Levina–Bickel intrinsic dimensionality estimate $m_{LB}$~\cite{Levina2005}, 
a proxy for local degrees of freedom; 
and (iii) the global participation ratio $PR_G$, which reflects 
the effective dimension of the entire feature cloud.
We now explain these metrics in detail.

\begin{figure*}[t]
    \centering
    \includegraphics[width=1\textwidth]{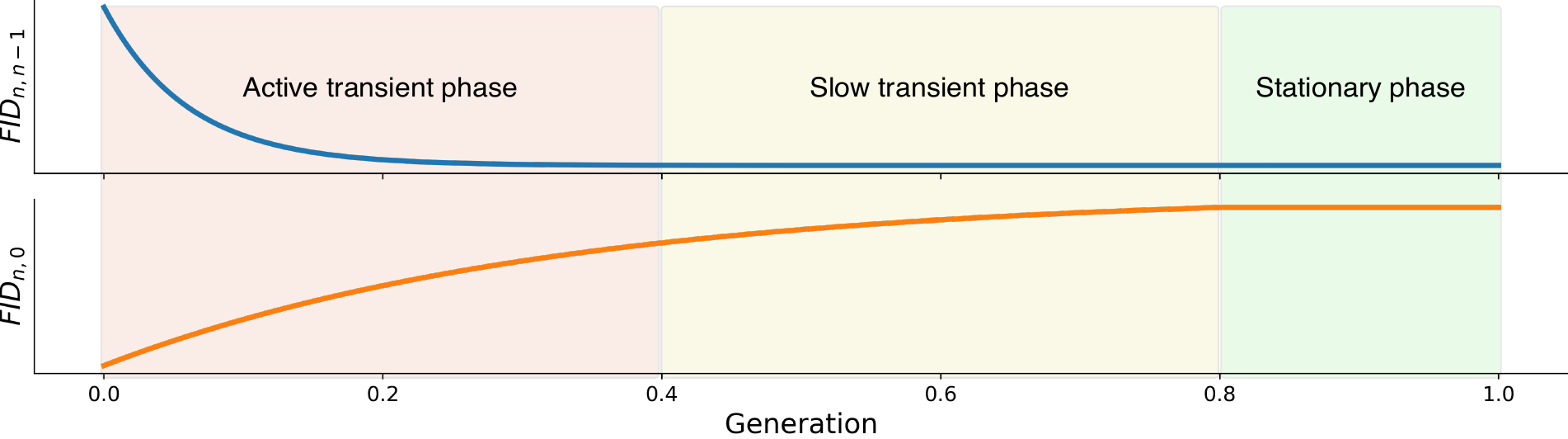}
    \caption{\textbf{A idealized example of local and cumulative drifts in generational Markov chains.} 
    The blue curve ($FID_{n,n-1}$) shows local drift between successive generations, 
    and the orange curve ($FID_{n,0}$) shows cumulative drift from the original data. 
    When both curves have steep slopes, the chain is in an \emph{active transient} phase; 
    when one curve flattens and the other changes slowly, it is in a \emph{slow transient} phase; 
    and when both plateau, the system has reached empirical stationarity.}
    \label{fig:drift-demo}
\end{figure*}
\subsubsection{Fréchet Inception Distance (FID)}

FID is the 2-Wasserstein distance between two Gaussians fitted to 
embedded features of two sample sets $S_r$ (real) and $S_g$ (generated),
\begin{equation}
    \text{FID} = \lVert \mu_r - \mu_g \rVert^2 + Tr \left( \Sigma_r + \Sigma_g - 2(\Sigma_r \Sigma_g)^{1/2} \right),
    \label{eq:fid}
\end{equation}
where $\mu_r, \Sigma_r$ and $\mu_g, \Sigma_g$ are the mean vectors and covariance matrices of the real and generated distributions, respectively.
The embeddings are extracted 
from a pre-trained network (usually Inception-V3~\cite{heusel17a} or a locally trained classification network).
Lower values indicate closer agreement between
real and generated distributions in latent space.

In our experiments, 
we compute FID scores across distributions generated at successive iterations. 
As an example,  suppose that $X_0$ is the original data distribution (real images) 
and $X_1, X_2, \dots X_N$ are the 
generated distributions at each generation $n$, i.e., the output 
of the generative samples obtained after the $n^{th}$ iteration. 
We estimate the FID score at the classification layer of Inception-V3 for ImageNet-5, 
and at convolution layers of pretrained networks for MNIST.
We track two complementary scores (Fig.~\ref{fig:drift-demo}):
\begin{enumerate}
    \item \textit{Local-drift} ($\operatorname{FID}_{n,n-1} = \operatorname{FID}_{X_n, X_{n-1}}$) measures 
      how much each new generation differs from the last.
    Smaller values indicate low drift across generations, consistent with local stationarity.
    \item \textit{Cumulative-drift} ($\operatorname{FID}_{n,0} = \operatorname{FID}_{X_n, X_0}$) estimates 
      the cumulative distance traveled from the original distribution.
\end{enumerate}
These two scores together
provide evidence consistent with empirical stationarity. As shown in Fig.~\ref{fig:drift-demo}, 
if both $\operatorname{FID}_{n,n-1}$ 
and $\operatorname{FID}_{n,0}$ have large negative/positive slope,  
the chain is in an \emph{active transient} phase; if one 
flattens while the other changes slowly, the chain is in a \emph{slow transient} phase;
if both flatten, the evidence points to empirical \emph{stationarity}.

\subsubsection{Intra‐Class Spread \texorpdfstring{$(\mathbf{\sigma_{intra}})$}{sigma-intra}}  
Intra-class root-mean square Euclidean deviation~\cite{fisher1936use}, or 
simply \emph{intra‐class spread}, measures the average 
local spread of latent representations 
within each semantic category, reflecting how spread or compact individual class clusters are. 
Formally, for $C$ classes with samples $X_c\subset\mathbb{R}^D$ and centroids $\mu_c$ of latent feature $x_L$, 
\begin{equation}
\sigma_{\mathrm{intra}}
= \frac{1}{C}\sum_{c=1}^{C}
  \sqrt{\frac{1}{|X_c|}\sum_{x\in X_c}\bigl\lVert x - \mu_c\bigr\rVert^2},\, \quad \text{where }\; \mu_c = \frac{1}{|X_c|}\sum_{x\in X_c} x_L.
  \label{eq:sigma-intra}
\end{equation}
Here, the term under the square‐root 
is the mean‐squared deviation of class-$c$ points from their centroid, 
and averaging over classes yields a single statistic capturing the typical cluster radius. 
$\sigma_{intra}$ provides a robust measure of local cluster 
behavior within each category, offering a gauge of 
how iterative processing reshapes class-specific features over time.

\noindent \textit{Behavior of $\mathrm{\sigma_{intra}}$:} 
It offers a granular view of how neural resonance reshapes the manifold 
in the vicinity of each data point.  
An increase in $\sigma_{intra}$ indicates \emph{local 
semantic expansion}, where cluster points diverge 
along latent axes across generations.
A decrease indicates \emph{semantic contraction}, 
where clusters tighten over generations.

\subsubsection{Intrinsic Dimensionality \texorpdfstring{$(\mathbf{m_{LB}})$}{mLB} }
Intrinsic dimensionality (Levina–Bickel Estimator~\cite{Levina2005}) 
captures the number of `free' directions in which data locally vary, i.e., 
the dimensionality of the manifold on which the high-dimensional points really lie. 
In a smooth latent manifold, each point $x_i$ lives in a tiny patch that is well‐approximated 
by an $m$-dimensional Euclidean neighborhood, so that the number of nearby samples 
within radius $r$ scales like $r^m$ $(Volume_{m}(r) \propto r^m)$. 
The expected number of points within radius $r$ is proportional to $r^m$;
Levina and Bickel’s method harnesses this scaling law by examining 
the ratios of the $k$-th nearest‐neighbor distance $T_i(k)$ to the distances $T_i(j)$ of the first $k\!-\!1$ neighbors,
\begin{equation}
  m_{\mathrm{LB}}=\frac{1}{n}\sum_{i=1}^n\hat m_i,  \,\quad \mbox{where}\,\,\,
\hat m_i \;=\; \biggl[\frac{1}{k-1}\sum_{j=1}^{k-1} \ln\frac{T_i(k)}{T_i(j)}\biggr]^{-1}.
\label{eq:levina-bickel}
\end{equation}
Here $n$ is the number of samples, $\hat m_i$ is the estimated dimension around point $i$, and $k$ is the neighbor count.
Averaging over all observations $n$ of a data cluster provides 
an estimate of local dimensionality. 
Though sensitive to number of neighbors $k$, the intrinsic dimensionality 
quantifies the manifold’s \emph{effective} number of independent axes: 
it is the local `degrees of freedom' available for data variation, 
and averaging $\hat m_i$ (Eq.~\ref{eq:levina-bickel}) over all points delivers a robust measure of 
the local latent representation’s geometric complexity. 

Intrinsic dimensionality is sensitive to folds\textemdash here 
termed `\emph{wrinkles}'\textemdash in the latent manifold. 
A \emph{wrinkle} is a localized, 
low‐amplitude perturbation that introduces subtle new axes of variation 
within a small neighborhood, thereby raising local dimensionality 
without substantially altering the global structure.
Geometrically, this represents small-scale deviations 
from the manifold’s smooth, low-curvature core\textemdash 
localized folds or nonlinearities that increase dimensionality in a restricted neighborhood. 
Rather than altering the manifold’s broad, global shape, 
wrinkles inject fine‐grained ripples that raise 
its local dimensionality without changing its overall coarse geometry.

\noindent \textit{Behavior of $\mathbf{m_{LB}}$: }
A decline in intrinsic dimensionality denotes semantic contraction, 
reflecting the manifold’s collapse onto a diminishing set of principal axes. 
In contrast, an increase in intrinsic dimensionality signals either 
genuine expansion into previously unoccupied subspaces or the emergence 
of localized `wrinkles' 
that transiently elevate neighborhood‐level degrees of freedom even amid local semantic contraction.

\subsubsection{Participation Ratio (PR)}  
The participation ratio quantifies the \emph{effective} number of principal axes 
that capture latent variance, serving as an alternate measure of manifold
dimensionality~\cite{halsey1986fractal,eckmann1988multifractal}.  
Given the eigenvalues $\{\lambda_i\}_{i=1}^D$ of the empirical covariance 
matrix of the latent features, participation ratio is,
\begin{equation}
\mathrm{PR}
= \frac{\bigl(\sum_{i=1}^D \lambda_i\bigr)^2}
       {\sum_{i=1}^D \lambda_i^2}\,.
\label{eq:pr}
\end{equation} 
Here $D$ is the total number of eigenvalues, typically $D$=$n$.
If variance is concentrated in a single direction, $\mathrm{PR}=1$; 
if it is evenly distributed across $k$ orthogonal 
directions, $\mathrm{PR}=k$.

The participation ratio can be estimated at different levels for an image distribution\textemdash
either at the class level, considering samples within each category, or at the global level, 
considering the entire distribution. In this work,
we compute global participation ratio $(\mathrm{PR}_G)$ from the 
full eigen-spectrum of latent features computed across the entire dataset.
$PR_G$ is a hyperparameter‐free, 
robust estimate of the manifold’s global `degrees of freedom'.

\noindent \textit{Behavior of $PR_G$:}

An increase in $PR_G$ denotes \emph{global expansion}, 
implying that variance has become more uniformly spread across a larger number of principal axes.
Conversely, a decrease in $PR_G$ signals \emph{global contraction}, as the latent representation progressively 
collapses onto a smaller subset of dominant directions. 
Participation ratio remains effectively invariant under small‐magnitude `wrinkles,' 
changing only when local perturbations inject enough 
variance into new directions to alter the global covariance spectrum.

\noindent \textit{Relationship between  $PR_G$ and $m_{LB}$: }
Although both metrics quantify dimensionality of a latent space, 
they capture complementary aspects of latent-space geometry.
$PR_G$ is derived from the full eigen-spectrum of the covariance matrix and therefore 
reflects the \emph{global} distribution of second-moment variance: if most variance is 
concentrated along only a few axes, $PR_G$ will be small even when the manifold 
locally meanders through many minor directions. By contrast, $m_{LB}$ is a maximum likelihood estimation 
based on how neighbor distances scale with radius and is thus sensitive to fine-scale 
curvature or wrinkles that may add local directions without materially affecting $PR_G$. 

The two metrics coincide on linear or nearly linear subspaces with uniformly 
distributed samples\textemdash such as a well-spread Gaussian blob\textemdash because under those conditions local 
scaling and global second moments capture the same structure, rising or falling together. 
They diverge, however, on strongly nonlinear manifolds. 
In short, $PR_G$ answers ``how many axes carry most of the variance for the whole distribution?''
whereas $m_{LB}$ asks ``how many directions does the manifold reveal locally?'' 
Used together, the two measures provide a fuller picture of latent-space 
structure than either can supply alone.

\begin{figure*}[t]
    \centering
    \includegraphics[width=1\textwidth]{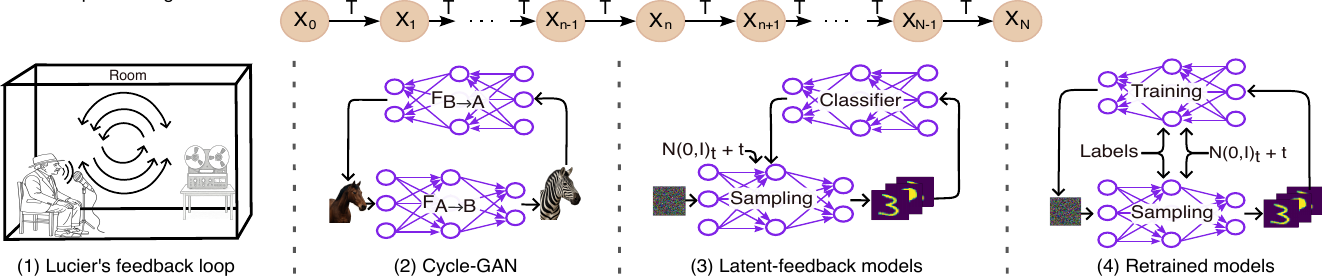}
    \caption{\textbf{A high-level representation of the Iterative feedback process.} $X_n$ 
    represents the current distribution of images at generation $n$ and $T(\cdot)$ is the 
    transformation operator. The framework unifies five feedback settings\textemdash 
    Lucier’s acoustic experiment, cyclic image translation (CycleGAN), latent-feedback diffusion, 
    and two retrained diffusion models (label-guided and unconditional)\textemdash each representing
    a distinct form of iterative feedback. Together they illustrate how successive generations of 
    models form a Markov process and their shared dynamics.}
    \label{fig:unified-generational-chain}
\end{figure*}
\section{Generative Feedback as a Markov Chain: Expanded}\label{sec:gmc}

We begin with an initial distribution of images ($X_0$). Repeated application 
of a transformation operator $T$ yields successive distributions $\{X_0, X_1, X_2,\dots\}$, 
where each generation depends only on its predecessor. This naturally defines a Markov 
chain across iterative generations.

We refer to this process as a \emph{generational Markov chain }(GMC).
Its characteristics are determined entirely by the properties of the operator $T$.  
If $T$ injects randomness, the chain explores broadly 
before settling into a stationary distribution; if $T$ is deterministic or lacks randomness,
the chain instead converges to a fixed point or local attractor. In this section, we present a detailed analysis 
of different iterative feedback frameworks shown in Fig.~\ref{fig:unified-generational-chain}
and explore their characteristics as a GMC.

Formally, let $X_0$ denote the initial distribution\textemdash an audio waveform 
in Lucier's feedback loop or an image distribution in generative models. The operator 
$T$ acts iteratively on this domain, mapping $T\colon \mathcal{X}\to\mathcal{X}$, where,
\begin{equation*}
    \mathcal{X} =
    \begin{cases}
    \mathbb{R}^d
      & \textit{Lucier’s Feedback}\\
    \mathbb{R}^{H\times W\times C}, 
      & \textit{CycleGAN}\\
    \mathbb{R}^{H\times W\times C}, 
      & \textit{Latent-Feedback}\\
    \{\,x \sim p_{\mathrm{data}}(x)\mid x \in \mathbb{R}^{H\times W\times C}\}, 
      & \textit{Uncond-Retrained}\\
    \{(x,y)\sim p_{\mathrm{data}}(x,y)\mid x\in\mathbb{R}^{H\times W\times C},\,y\in\{1,\dots,K\}\}.
      & \textit{Cond-Retrained}
    \end{cases}
    \label{eq:operator-details}
\end{equation*}
In this context, $\mathbb{R}^{H\times W\times C}$ represents the space of images 
with spatial dimensions $H\times W$ and $C$ color channels. 
Here, $x \in \mathbb{R}^{H\times W\times C}$ denotes a single image sample. 
and $y$ represents 
its associated class label with $K$ distinct classes. 
The distribution $p_{\mathrm{data}}(\cdot)$ characterizes the empirical 
probability distribution of the dataset.
For each iterative step $X_{n+1} = T (X_n)$, operator $T$ is
\begin{equation}
    \begin{split}
        T &=
        \left\{
        \begin{array}{@{}l@{\quad}l@{\quad}l@{}}
        h * (\cdot), 
          & \textit{Lucier’s Feedback Loop}\\[2pt]
        F_{A\to B}\bigl(F_{B\to A}(\cdot)\bigr), 
          & \textit{CycleGAN}\\[2pt]
        D(\cdot\mid f(\cdot)), 
          & \textit{Latent-Feedback Models}\\[2pt]
        \mathrm{Sample}(\mathrm{Train}(\cdot)), 
          & \textit{Un/Conditional Retrained Models}
        \end{array}
        \right.
    \end{split}
\end{equation}
where $h$ denotes a convolutional kernel applied iteratively in 
the acoustic feedback scenario \cite{abel2012luciverb}; 
$F_{A\to B}(F_{B\to A}(\cdot))$ constitutes a cyclic operator implemented 
via paired generative adversarial networks (GAN) over two domains $A$ and $B$, $F_{A\to B}$ and $F_{B\to A}$; 
$D(\cdot\mid f(\cdot))$ symbolizes latent-feedback sampling 
with a pre-trained generative diffusion model $D$ guided by 
features extracted from a pre-trained classification network $f(\cdot)$; and finally, 
$\textit{Sample}(\textit{Train}(\cdot))$ describes an iterative 
process of repeated model retraining and subsequent DDPM (denoising diffusion probabilistic models) 
sampling \cite{ho2020denoising}. 
This framework captures the common mechanism that underlies diverse 
collapse scenarios\textemdash including the `model collapse' phenomenon~\cite{shumailov2024ai}, 
which we refer to as \emph{generational collapse}\textemdash and, under the conditions 
detailed in Section \ref{sec:neural-resonance}, unifies them within the concept of \emph{neural resonance}.

Generational collapse occurs when repeated application of a 
transformation operator amplifies structural or semantic preferences. 
At the same time, it suppresses semantic variance in the generated distribution. 
Whether arising from architectural or model biases, or 
from retraining-induced biases, 
each iteration pushes the generated distribution of GMCs closer to a low-dimensional attractor state by
altering the effective feature support of the output distribution in latent space. 
Ultimately, the GMCs converge on this attractor state, producing distributions that 
remain unchanged under further iterations.

\subsection{Proof of Markov Property}\label{subsec:as-markov-chain}

We model this iterative framework as a Markov process
over generations.
Figure~\ref{fig:unified-generational-chain} illustrates a generational Markov chain, 
where the repeated application of the operator $T(\cdot)$ yields $X_{n+1} = T(X_n)$. 
The diagram depicts this progression: from the initial distribution $X_0$ 
through successive generations $X_1$, $X_2$, $\dots$, ultimately approaching 
a stationary distribution $X_N$.  

All four experimental settings\textemdash functional analogue of Lucier's feedback loop, 
CycleGAN translation, latent-feedback diffusion, and retrained diffusion\textemdash 
fit this framework as iterative operators acting on a state space $\mathbb{E}$.
We analyze the \emph{Markovian} nature of the iterative framework, i.e., 
a process in which the next state depends only on the present state, not the 
full history. Formally, a Markov chain is a stochastic 
process $\{X_n\}$ that satisfies the property:
\begin{equation}
    P(X_{n+1}\mid X_n,X_{n-1},\dots,X_0) = P(X_{n+1}\mid X_n).
    \label{eq:markov-property}
\end{equation}

\noindent \textbf{Lucier's feedback loop:}
For the functional analogue of Lucier's feedback loop, let the audio waveform 
at iteration $n$ be $X_n$. The update is 
\begin{equation}
    X_{n+1} = T(X_n) \Rightarrow  X_{n+1} = h*(X_n),
\end{equation}
where $h$ is a fixed convolution kernel and $*$ denotes the one-dimensional convolution operation. 
Since $T(\cdot)$ is fixed, the Markov property holds directly. 

\noindent \textbf{CycleGAN:}
For CycleGAN, an image at generation $n$ is mapped as 
\begin{equation}
    X_{n+1} = T(X_n) \Rightarrow X_{n+1} = F_{A \to B} (F_{B \to A}(X_n)),
\end{equation}
where $F_{A \to B}: A\to B$ and $F_{B \to A}: B\to A$ are fixed neural network mappings 
and $X_n$ is an image sample.
Since the operator is deterministic, the Markov property is satisfied.

\noindent \textbf{Latent-feedback models:}
For latent-feedback diffusion, each iteration samples from a fixed diffusion model 
$\mathcal{D}$, conditioned on features from a deterministic extractor $f(\cdot)$:
\begin{equation}
X_{n+1} = T(X_n) \Rightarrow X_{n+1} = \mathcal{D}(\cdot \mid f(X_n)),
\end{equation}
where, $X_n$ is a distribution. Because each step depends only on the previous latent state, the process is Markovian.

\noindent \textbf{Retrained diffusion models:}

In retrained diffusion models, each new distribution $X_{n+1}$ depends 
not only on $X_n$ but also on the model parameters $\mathcal{D}_{n+1}$. 
Thus, at the distribution level, the Markov property does not hold. 
However, if we expand the state to include both the parameters and samples\textemdash
$(\mathcal{D}_n,X_n)$\textemdash the chain becomes Markovian:
\begin{equation}
    P((\mathcal{D}_{n+1}, X_{n+1})\mid(\mathcal{D}_n,X_n),\dots) = P((\mathcal{D}_{n+1}, X_{n+1})\mid(\mathcal{D}_n,X_n)).
\end{equation}
Hence retrained models admit a Markovian representation in the joint parameter-sample space.
The operator $T$ can be applied at a sample level or at distribution level; for simplicity, we 
treat $X_n$ as a distribution hereafter. 

\subsection{Ergodicity of GMCs: Sketch Proof}
\label{subsec:gmc-ergodic-intuition}

Before giving the full technical proof in \S3.3, we sketch the main ideas behind the
ergodic behavior of generational Markov chains. In what follows, $X_n$ is 
considered as a distribution, unless stated otherwise. We provide the sketch proof 
and the expander proof in \S3.3 
for diffusion-based GMCs, but it can be extended to other GMCs with different state space.

\noindent \textbf{Setup:}
The state of the GMC at generation $n$ is a distribution $X_n$ over the pixel cube
$\mathbf{E} = [0,1]^{H \times W \times C},$
that is, all images of shape $H \times W \times C$ with pixel values in $[0,1]$.
A single generational update is implemented by running reverse-DDPM sampling
procedure (Alg.~\ref{algo:ddpm-sampling}) conditioned on features $f(X_n)$:
starting from Gaussian noise, the sampler repeatedly, 
adds Gaussian noise (with full support), and 
applies a smooth neural network update that nudges the sample toward
high-density regions,
and finally returns new images $x_i \in X_{n+1}$ in $\mathbf{E}$ \cite{ho2020denoising}.
This defines a Markov chain $X_{n+1} = T(X_n)$, where $T$ is the
one-generation transition rule. We claim that the diffusion-based generational chains are ergodic: in the
long run it forgets its starting point and converges to a unique long-run
distribution over images. We define this property in terms of continuous state space Markov chains.

The argument relies on three intuitive properties of the dynamics of reverse-diffusion: (i) \textit{Everywhere-noisy transitions},  In Alg.~\ref{algo:ddpm-sampling},
at each reverse step, we inject Gaussian noise whose density is strictly
positive everywhere in $\mathbf{E}$. No region is ever assigned zero
probability by the noise alone.
(ii) \textit{Smooth neural updates},  
the neural network in the sampler is built from dot products and standard activations
(ReLU, sigmoid, $\tanh$), so small changes in its input produce small
changes in its output almost everywhere. The overall DDPM update is a
smooth deformation of the injected noise. Finally, (iii) \textit{Bounded image space},  
all generations are kept inside a fixed bounded region:
pixel values are clipped into $[0,1]$, so the chain moves within a
compact set of possible images and cannot escape to infinity. For a continuous state-space Markov chain, ergodicity requires showing three properties: irreducibility, aperiodicity,
and repeated mixing in a bounded space. We can use the intuitive properties of reverse-diffusion to prove that the generational Markov chain has these properties.

\noindent \textbf{Irreducibility:}  
A Markov chain is irreducible if a distribution over a 
given state space assigns nonzero probability mass across the entire state space in the limit. Consider one full generational update, from $X_n$ to $X_{n+1}$.
We first sample a Gaussian noise vector with nonzero probability everywhere,
then pass it through a smooth neural network plus clipping to $[0,1]$.
Because the noise covers the full space and the network does not collapse
everything to a lower-dimensional set, the resulting probability distribution of
$X_{n+1}$ assigns \emph{nonzero} probability to every region of the image cube
$\mathbf{E}$. Intuitively, there is always some (possibly very small) chance of
moving from any given image into any non-negligible region of the image space
in a single transition step.
This is the Markov-chain notion of being able to ``eventually reach everywhere''
from anywhere, and is the intuition behind irreducibility. For more details please see
Sec.~\ref{subsubsec:irreducibility}.

\noindent \textbf{Aperiodicity:}  
A Markov chain is aperiodic if it does not return to a region within the state space at a regular interval (e.g., only return to a state or
region after exactly 3, 6, 9, \dots generations). Because the transition from $X_n$ to $X_{n+1}$ is noisy and smooth, the chain
can not only move far away but also move around locally.
Starting from a given state $X_n$, there is a strictly positive probability that
the next state stays within a tiny neighborhood of $X_n$.
If the chain was periodic, such local randomness would be
impossible: there would be generations where it \emph{must} move elsewhere.
The presence of noise at every step breaks such rigid rhythms.
This rules out fixed periods and gives the intuition for aperiodicity.

\noindent \textbf{Repeated mixing in a bounded space:}  
The chain evolves in a bounded image cube and every one-step transition has a
small but strictly positive probability of landing in any given region of that
cube. In other words, each generation performs a random, noisy `mixing move'
that never completely ignores any subset of images. Mixing leads to
a strong notion of recurrence: the chain keeps coming back to every region
with positive probability.

A Markov chain with these properties
(irreducibility, aperiodicity, and strong mixing in a bounded state space) behaves like a well-mixed
random walk.
The next sections formalize this sketch in the language of continuous state-space Markov chains and provide a full
measure-theoretic proof.

\subsection{Ergodicity of GMCs: Expanded Proof}
\label{subsec:ergodicity-expanded-proof}
 
We illustrate the characteristics of GMCs using the latent-feedback diffusion model, 
though the results extend to all GMCs within our framework.
The GMC operates in the state space $\mathbf{E} = [0,1]^{H\times W \times C}$ endowed with 
Borel $\sigma$-field, meaning random variable $x_n$ takes values in $\mathbf{E}$ and 
are measurable with respect to that $\sigma$-field and the associated 
transition kernel $K(x, \cdot)$ induced by $T$.
The transition kernel assigns a probability measure on $(\mathbf{E}, \mathcal{B}(\mathbf{E}))$ for each state $x$, 
enabling us to perform integration, and to define expectations for the process.
We adopt the Lebesgue\footnote{Lebesgue on the cube means the usual notion of volume 
in $d$-dimensions (length in 1D, area in 2D, volume in 3D, and so on), 
restricted to the cube $[0,1]^{H \times W \times C}$. This provides a natural 
baseline measure against which probabilities and densities on the state space can be defined.} 
measure on the bounded cube $[0,1]^{H \times W \times C}$
as the reference measure, ensuring $\sigma$-finiteness.
This implies that every one-step law $K(x, \cdot)$\textemdash the probability 
distribution of the next state given the current state,
of the Markov chain is continuous 
with respect to the measure $\psi$ and allow us to write each transition probability in density ($k(x,y)$) form,
\begin{equation}
    K(x,A) = \int_A k(x,y) \; \psi(dy), \quad x \in \mathbf{E}, A \subseteq \mathbf{E},
\end{equation}
where $\psi$ never assigns a zero value to any $K(x, \cdot)$.
The per-generation transition kernel, $K: \mathbf{E} \times \mathcal{B}(\mathbf{E}) \rightarrow [0,1]$ is 
defined as 
\begin{equation}
    K(x,B) = \int_B P_{\theta}(y | f(x))\; \psi(dy), \quad B \in \mathcal{B}(\mathbf{E})
\end{equation}
where $P_{\theta}(\cdot)$ is the one-step output density. Equivalently, at the level of 
random variables, 
\begin{equation}
    X_{n+1} \sim P_{\theta}(\cdot | f(X_n)), \quad n = 0,1,\dots
\end{equation}
and the generalized GMC is written as $\{ X_n\}_{n \geq 0} : X_{n+1} \sim K(X_n, \cdot)$, 
where $K$ is the transition kernel induced by $T$ as shown in Fig. \ref{fig:unified-generational-chain}.

The latent-feedback generative models uses a pre-trained image classification
network $f(.): \mathbf{E} \rightarrow \mathbb{R}^d$, to provide $d$-dimensional 
conditional feature vectors as conditional to the diffusion model. 
We follow Dhariwal and Nichol ~\cite{dhariwal2021diffusion} for network architecture 
and diffusion noise schedules. We also use Langevin dynamics~\cite{song2019generative} in 
Denoising Diffusion Probabilistic Model (DDPM)~\cite{ho2020denoising} for 
sampling images. We briefly describe the DDPM sampling process as it is an integral part 
of our proof.

The DDPM sampling stage (see Alg. \ref{algo:ddpm-sampling}) can be viewed as a \textit{discrete Langevin-type dynamics} run in reverse time.  
Starting from pure Gaussian noise $x_T\sim\mathcal N(0,I)$, each reverse step adds a 
drift term proportional to the learned score $\nabla_x\log p_\theta(x_t)$\textemdash pulling 
the sample toward high-density regions\textemdash and injects \textit{annealed} Gaussian noise 
(see step 4 of Alg. \ref{algo:ddpm-sampling})
whose variance shrinks with timestep $t$. 
The learned score field approximates the score $\nabla_x\log p_t(x_t)$~\cite{song2020score},
while the gradually decreasing noise scale $\sigma_t$ plays the role of the temperature schedule, 
yielding a stochastic but guided trajectory that transforms 
white noise into a realistic image as $t$ marches to $0$. The annealed Gaussian 
step guarantees a strictly positive one-step density under mild conditions.

With the above description of the generational Markov chain, we argue that the chain is 
$\psi$- irreducible, aperiodic, and positive Harris recurrent in a 
compact state space $\mathbf{E} = [0,1]^{H \times W \times C}$. We base our proof on 
time-homogeneous kernels as required by Meyn-Tweedie theories~\cite{meynandtweedie2009} 
and can be extended to time-inhomogeneous kernels.

\begin{algorithm}[t]
\caption{DDPM Sampling Process (reprinted from Ho et al.~\cite{ho2020denoising})}
\label{algo:ddpm-sampling}
\begin{algorithmic}[1]
\State $\mathbf{x}_T \sim \mathcal{N}(0,I)$
\For{$t = T,\dots,1$}
    \State $\mathbf{z} \sim \mathcal{N}(0,I)$ \textbf{if} $t>1$ \textbf{else} $\mathbf{z}=0$ \Comment{\textcolor{gray}{: Sample Gaussian noise}}
    \State $\displaystyle 
      \mathbf{x}_{t-1}
      = \frac{1}{\sqrt{\alpha_t}}
        \Bigl(
          \mathbf{x}_t
          - \frac{1-\alpha_t}{\sqrt{1-\bar\alpha_t}}\,
            \varepsilon_\theta(\mathbf{x}_t,t)
        \Bigr)
        + \sigma_t\,\mathbf{z}$ \Comment{\textcolor{gray}{: Annealed Gaussian step}}
\EndFor
\State \textbf{return} $\mathbf{x}_0$
\end{algorithmic}
\end{algorithm}

\subsubsection{\texorpdfstring{$\psi$}{psi}-Irreducibility}
\label{subsubsec:irreducibility}

Irreducibility means the chain is not confined: from any state, it 
can eventually reach every other region of the space with non-zero i.e., \textit{everywhere noisy transition}.
A Markov chain is irreducible when every state can reach every other state 
with positive probability in some finite number of steps i.e., the ability of a Markov chain to reach 
any non-null set in finite time. Formally, it is defined as:
\begin{definition}[$\psi$-irreducible]
    A kernel $K$ is $\psi-irreducible$ if for every $x \in \mathbf{E}$ and 
    every measurable set $A$ with $\psi(A) > 0$, there exists an $n \geq 1$ such that the n-step kernel $K^n(x,A) > 0$. 
\end{definition}

\begin{lemma}[One-step positive density]
\label{lemma:one-step-density}
    One-step diffusion-based generative Markov kernel has positive density; i.e., 
    the one-step density $K(x,y)$ of $K(x, \cdot)$ satisfies $K(x,y) > 0$ for all $y \in \mathbf{E}$. 
    Therefore, for any $A$ with $\psi(A) > 0$, 
    \begin{equation}
        K(x,A) = \int_A K(x,y) \; \psi (dy) > 0.
        \label{eq:positive-density}
    \end{equation}
\end{lemma}

\begin{proof}
    Inside the one outer call to the diffusion model, we simulate DDPM sampling from i.i.d. Gaussian 
    noise. As shown in Alg. \ref{algo:ddpm-sampling}, it draws $z_T \sim \mathcal{N}(0,I)$, which has 
    positive density everywhere, followed by a smooth transport 
    map $T^c_{\theta}: \mathcal{R}^{d_z} \rightarrow \mathbf{E}$ defined by the 
    neural network.

    The Gaussian has a positive density and the neural network\textemdash which, with smooth activation 
    functions (e.g. sigmoid, tanh), defines a mapping that is continuous  
    almost everywhere for piecewise-linear functions (e.g. ReLU)\textemdash 
    is continuous almost everywhere, given standard neural activations (e.g., ReLU, Sigmoid, tanh). 
    The one-step-density $K(x,y)>0$ for all $y \in \mathbf{E}$.
    Therefore, Eq. \ref{eq:positive-density} holds for any $A$ with $\psi(A)>0$. 
    Thus, $\psi$-irreducibility follows directly with $n=1$(similar arguments are used in DDPM~\cite{ho2020denoising}). 
\end{proof}

In the proof, we make the following two assumptions which are 
usually satisfied in practice: 
\begin{enumerate}
    \item \textit{Surjectivity of $f(x)$:} Even if the feature extractor maps all images in a region 
    to the same class label\textemdash so that, from the model’s conditional prior, this region has zero density 
    and the differences between images are invisible\textemdash the Gaussian noise injected at each sampling 
    step has full support (positive density everywhere). This ensures that the Markov chain can still move anywhere in the state 
    space, preserving irreducibility.
    \item \textit{Bounded outputs of neural networks:} Because $T^c_{\theta}$ is continuous (and often 
    Lipschitz), small changes in the input produce only small changes in the output. This prevents 
    the mapping from sending probability mass to pathological points like $\pm \infty$ or collapsing 
    into Dirac spikes. As a result, the transformed density remains smooth and strictly positive.
\end{enumerate}
The irreducibility proof only fails if the class-conditional sampler 
truly assigns zero probability to some positive-volume 
region of the pixel space.

\subsubsection{Aperiodicity}\label{subsubsec:aperiodicity}

Aperiodicity in Markov chains means the chain does not evolve in fixed cycles\textemdash
returns can occur at irregular times rather than being locked to a fixed rhythm.
Starting from any state, there is no fixed rhythm\textemdash such as `every third step' or `only on 
even times'\textemdash that dictates when the chain can return to a set of states with positive probability.
Formally, for a Markov chain with transition kernel $K$ on a measurable 
state-space ($\mathbf{E}, \mathcal{B}(\mathbf{E})$) that is $\psi$-irreducible for a measure $\psi$, 
the period of this chain is given as 
\begin{equation}
    d = gcd\{ n \geq 1: K^n(x,A) > 0 \text{ for some $A$ with }K(x,A) > 0\}
    \label{eq:period-of-chain}
\end{equation}
where $K^n(x,A)$ is the probability of landing in $A$ 
after exactly $n$-steps from $x$. 
The term `$gcd$' is the greatest common divisor of all the step-counts and represents the period ($d$) of the 
chain. The chain becomes aperiodic if $d=1$, i.e., if there is no clock that forces the chain to 
move in a fixed rhythm.

\begin{lemma}[Local self-transition]
\label{lemma:local-self-transition}
    If for every $x \in \mathbf{E}$ there exists a neighborhood $\mathcal{B}_{\mathcal{\varepsilon}}(x)$ such that
    $K(x, \mathcal{B_\varepsilon}(x))>0$, the chain is aperiodic; i.e., if $\forall x \in \mathbf{E}, \; \exists \;\varepsilon > 0: K(x, \mathcal{B_\varepsilon}(x)) > 0$, then the chain is aperiodic. 
\end{lemma}
\begin{proof}
Lemma \ref{lemma:one-step-density} shows that the one-step density $K(x,y)$ is strictly 
positive everywhere, $K(x,y)>0, \forall \; x,y \in \mathbf{E}$. 
For any state $x$, a small neighborhood ball $\mathcal{B_\varepsilon}(x)$ has non-zero
transition probability, ensuring aperiodicity. 
\begin{equation*}
    K(x, \mathcal{B_\varepsilon}(x)) = \int_{\mathcal{B_\varepsilon}(x)} K(x,y) \; \psi(dy) \; >0,
\end{equation*}
and hence, aperiodic. 
\end{proof}
We can also prove aperiodicity by contradiction, suppose the period $d>1$. Then, all non-zero return probabilities 
would occur at step $d, 2d, 3d, \dots$. But, the Lemma \ref{lemma:one-step-density} says there is already a
positive chance to come back in one-step, contradicting periodicity. Therefore, the chain is aperiodic 
with period $d=1$. Aperiodicity leads to smooth neural updates in GMCs.

\subsubsection{Harris Recurrence}\label{subsubsecc:positive-reccurrence}

In a Markov chain, Harris recurrence means that starting from any state, the chain will
return to any set of states with positive measure infinitely often, with probability 1. 
It guarantees the chain will eventually revisit any significant portion of its state
space repeatedly. We use petite-set to establish this property for generational Markov chains. 

A \emph{petite set} is a subset of the state space that, relative to the Markov kernel, ensures 
the chain has enough chance to mix, which is crucial for establishing ergodicity. 
In other words, petite sets act as `anchors' where the chain does not get stuck and retains 
access to the rest of the space.\footnote{By contrast, a finite set only describes the number 
of elements; it does not guarantee such mixing. For instance, a finite set may fail to be petite 
if the chain gets trapped in an absorbing state, e.g. CycleGAN GMC}
In our DDPM kernel, the Gaussian noise term ensures that 
the whole image cube is petite, even though it is uncountable.
Typical images lie in the bounded pixel cube $[0,1]^{H\times W \times C}$, and the 
overall state space is finite-dimensional.
The reverse-DDPM transition 
density function is jointly continuous and strictly positive, making 
the entire bounded cube qualifies as a petite set. The proof 
proceeds as follows.

Assuming $\psi$-irreducibility (Lemma \ref{lemma:one-step-density}) 
and aperiodicity (Lemma \ref{lemma:local-self-transition}) of our GMC with finite state space 
$\mathbf{E} = [0,1]^{H \times W \times C}$ ($\subset \mathbb{R}^{784}$ for MNIST), 
we apply petite set proof to prove 
Harris recurrence. To this end, we restate small set condition 
and petite set Theorems 5.5.3 from Meyn and Tweedie, 2nd edition and refer to the book
for their proofs.

\begin{lemma}[Small set in finite space]
\label{lemma:small-set}
    A set $C \in \mathcal{B}(\mathbf{E})$ is called a small set if there exits 
    an $m>0$, and a non-trivial measure $v_m$ on $\mathcal{B}(\mathbf{E})$, such that for all 
    $x \in C, \; B \in \mathcal{B}(\mathbf{E})$, 
    \begin{equation*}
        K^m (x, B) \geq v_m (B).
    \end{equation*}
    Then, $C$ is $v_m$-small.
\end{lemma}
\noindent We refer to the proof provided in Sec. 5.2, Meyn and Tweedie, 2nd edition~\cite{meynandtweedie2009}. 

\begin{lemma}[Petite set minorization]
\label{lemma:petite-sets}
A set $C \in \mathcal{B}(\mathbf{E})$ $v_a$-petite
if the sampled chain satisfies the bound 
    $K_a(x,B) \geq v_a(B)$, for all $x \in C, \; B \in \mathcal{B}(\mathbf{E})$, where $v_a$ is a 
    non-trivial measure on $\mathcal{B}(\mathbf{E})$, stated as: 
    \begin{equation*}
        \text{If }\; C \in \mathcal{B}(\mathbf{E})\; \text{is } v_m \;\text{-small, then } C \text{is }\; v_{\delta_m} \text{-petite.}
    \end{equation*}
\end{lemma}
\noindent We refer to Sec. 5.5.2 in Meyn and Tweedie (2nd edition) for its proof~\cite{meynandtweedie2009}.

The reverse DDPM transition density $k(x,y)$ is jointly continuous and strictly positive on the
compact set $\mathbf{E}\times\mathbf{E}$. 
Because $k(x,y)$ is continuous and $\mathbf{E}\times\mathbf{E}$ is compact, 
$k$ cannot shrink arbitrarily close to zero without actually reaching a minimum value $\varepsilon$. 
Thus, there exists a constant $\varepsilon>0$ such that $k(x,y)\ge \varepsilon$ 
for all $x,y\in\mathbf{E}$.
\begin{equation*}
    \varepsilon \;:=\; \min_{(x,y)\in \mathbf{E}\times \mathbf{E}} k(x,y) \;>\; 0.
\end{equation*}
Therefore, for any measurable $B\subseteq \mathbf{E}$ and all $x\in\mathbf{E}$,
\begin{equation}
K(x,B) \;=\; \int_B k(x,y)\,\psi(dy) \;\ge\; \int_B \varepsilon\,\psi(dy) 
\;=\; \varepsilon\,\psi(B),
\label{eq:petite-set-minorisation}
\end{equation}
which yields the desired petite set minorization $K(x,\cdot)\ge \varepsilon\,\psi(\cdot)$.
Therefore, $\mathbf{E}$ is a small set and using Lemmas \ref{lemma:small-set} and \ref{lemma:petite-sets}, 
it is automatically petite. 
In our DDPM kernel, the Gaussian noise term ensures that 
the whole image cube is petite, even though it is uncountable.
In other words, the whole cube ($[0,1]^{H \times W \times C}$) functions 
as a regeneration region, ensuring repeated access to the full state space.

Harris recurrent is the right notion of ``will certainly return'' for Markov chains in a general 
state space. If a Markov chain is $\psi$-irreducible, aperiodic, and has a petite set, then the chain is 
Harris recurrent (Meyn and Tweedie, 2nd edition, Theorem 9.1.7). We restate the theorem here 
and refer to the book for the proof. 

\begin{lemma}[$\psi$-irreducible + aperiodicity + petite set $\Rightarrow$ Harris recurrence] 
\label{lemma:harris-recurrece}
    Suppose that a Markov chain is $\psi$-irreducible, 
    if there exists some petite set in $\mathcal{B}(\mathbf{E})$
    such that the hitting probability $L(x,C) \equiv 1, \; x \in \mathbf{E}$, then the chain is Harris recurrent.
\end{lemma}
\noindent In other words, if there is any petite set $C$ that the chain is guaranteed to hit eventually, 
no matter where it starts, i.e., hitting probability $L(x,C) \equiv 1$ for every state $x$, then 
the whole chain is Harris recurrent. Harris recurrence requires that, relative to the same $\psi$, every 
set of positive $\psi$-measure is hit with probability 1 from every starting point. 

\subsubsection{Ergodicity}\label{subsec:ergodicity}

We now complete the analysis of the generational Markov chain (GMC) by proving
that the latent-feedback GMC is \emph{ergodic}: irrespective of the starting state, 
it converges to a unique invariant distribution. 
We begin by recalling the relevant metric.

\begin{definition}[Total variation (TV) norm ($\lVert \cdot \rVert_{TV}$)] if $\mu$ is a signed measure on $\mathcal{B}(\mathbf{E})$, then 
    the total variation norm $\lVert \mu \rVert$ is defined as 
    \begin{equation*}
        \lVert \mu \rVert := \sup_{f:\vert f \vert \leq 1} \vert \mu(f) \vert = \sup_{A \in \mathcal{B}(\mathbf{E})} \mu(A) - \inf_{A \in \mathcal{B}(\mathbf{E})} \mu(A)
    \end{equation*}
\noindent The total variation (TV) norm measures the maximum 
discrepancy between the probability measures.
There are two equivalent ways to measure this total size: one is to see how large $\mu$ can be 
when integrated against any test function $f$ with values between $-1$ and $1$; 
equivalently, it measures the difference between the sets where the probability
mass is maximally assigned and minimally assigned. 
The TV distance defined this way is a genuine metric 
between probability measures and always lies between $0$ and $1$.
\end{definition}
Recall from Lemma~\ref{lemma:petite-sets} that the one–step density of the GMC satisfies a
petite set condition (Eq. \ref{eq:petite-set-minorisation})
for some $\varepsilon>0$ and the $\sigma$-finite reference measure~$\psi$.  In
combination with $\psi$-irreducibility (Lemma~\ref{lemma:one-step-density}), aperiodicity
(Lemma~\ref{lemma:local-self-transition}), and 
positive Harris recurrence (Lemma~\ref{lemma:harris-recurrece}), this is enough to
establish ergodicity of the GMC.

\begin{corollary}[Ergodic GMCs]
\label{cor:ergodic}
The generational Markov chain $\{ X_n\}_{n \geq 0} : X_{n+1} \sim K(X_n, \cdot)$, with transition kernel $K$, is ergodic on state space
$E\subset[0,1]^{H\times W\times C}$ and have following properties:
\begin{enumerate}
    \item \textit{Existence and uniqueness of the stationary law.}  
    If an ergodic Markov chain starts from an initial distribution $X_0 \sim \Pi_0$, it converges to 
    a unique stationary distribution $\Pi$ after mixing and this distribution $\pi$ remains invariant 
    under further application of the transition kernel $K$.
  \item \textit{$K^n(x, \cdot)$ converges to $\pi$ in total variation}  
        from a fixed starting point $x\in \mathbf{E}$ and $n\ge 0$,
        \begin{equation}
            \lim_{n \rightarrow \infty} \lVert K^{n}(x,\cdot)-\pi\rVert_{\mathrm{TV}} = 0, \quad \forall \; x \in \mathbf{E}.
        \end{equation}
\end{enumerate}
\end{corollary}

\begin{proof}
  $\psi$-irreducibility together with aperiodicity ensures that any two
  invariant probability measures must coincide
  (Meyn–Tweedie, Theorem 13.0.1), hence uniqueness once existence is shown.
  Positive Harris recurrence (Lemma~3.5) furnishes existence.
\end{proof}

\begin{figure*}[t]
    \centering
    \includegraphics[width=1\textwidth]{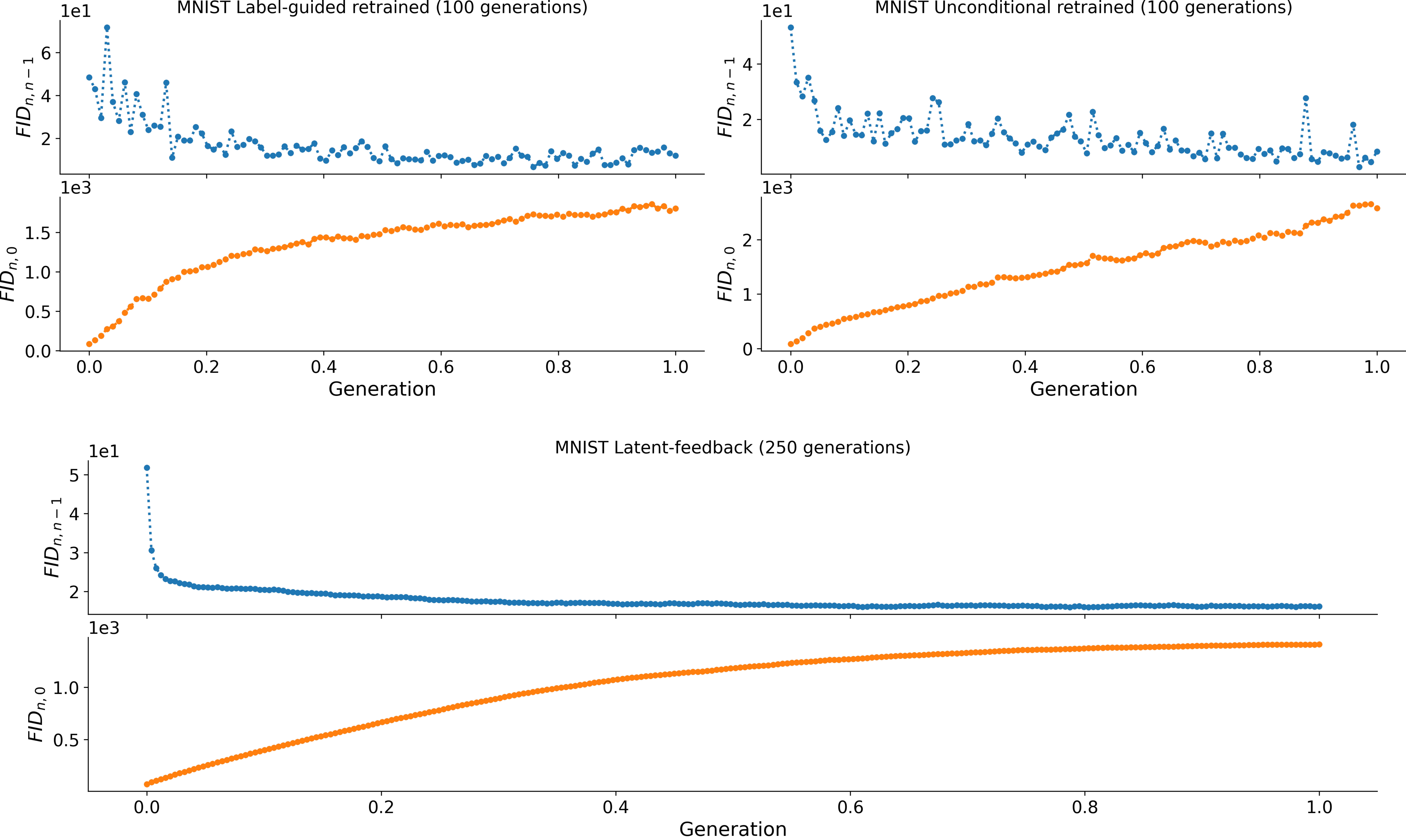}
    \caption{\textbf{Local and cumulative drifts in generational Markov chains on MNIST.} 
    Each plot shows the local drift ($FID_{n,n-1}$, top) and cumulative drift ($FID_{n,0}$, bottom)
    across generations. Results are shown for the label-guided retrained diffusion model (top left), 
    the unconditional retrained diffusion model (top right), 
    and the latent-feedback diffusion model (bottom). 
    Together they illustrate distinct convergence behaviors\textemdash rapid stabilization 
    in the label-guided case, continued evolution in the unconditional model, 
    and gradual stationarity in the latent-feedback chain.}
    \label{fig:mnist-drift-plots}
\end{figure*}
\subsection{Empirical Analysis of GMCs}

We empirically analyze both ergodic and non-ergodic GMCs. A 
chain is classified as ergodic if it satisfies three 
conditions: $\psi$-irreducibility, aperiodicity, and Harris recurrence.
For our experiments\textemdash particularly in the 
context of iterative feedback for image generative models\textemdash 
the key requirement in practice is the one-step positive density condition, 
which ensures a non-zero probability of moving between any two states.
In practice, this condition can 
be satisfied by sampling algorithms that incorporate an annealed Gaussian step 
(Gaussian noise step, Algo.~\ref{algo:ddpm-sampling}). Conversely, any GMC that 
fails to meet one or more of the above conditions is non-ergodic.

\subsubsection{Non-Ergodic GMCs}\label{subsubsec:non-ergodic}

The construction of most iterative feedback-based image generative chains 
follows a common pattern: samples generated by the current model are used to 
train the next model. 
The functional analogue of Lucier’s feedback loop 
and CycleGAN chains fail this one-step density requirement, 
and thus are non-ergodic.
Both systems instead behave like absorbing Markov chains with a 
finite set of states. 

The functional analogue of Lucier’s feedback loop is modeled using a fixed 
room impulse response, which repeatedly applies convolution to its own output. 
As a result, the induced GMC lacks a positive transition probability from one 
state to any other. Similarly, the CycleGAN GMC exhibits the behavior of an 
absorbing Markov chain with multiple attractor states. During the cyclic 
translation process between domains, output images move from an attractor in 
one domain to the corresponding attractor in the other (Fig. 2 of the main paper). 

Although they may appear stationary, these GMCs simply cycle among 
a finite set of absorbing states. Moreover, local and cumulative drifts 
could not be reliably estimated because outputs quickly drifted 
out-of-domain for their respective classifier networks.

\subsubsection{Ergodic GMCs}\label{subsubsec:ergodic}

All other GMCs in our diffusion-based framework\textemdash namely latent-feedback, 
label-guided retrained, and unconditional retrained\textemdash satisfy the critical 
one-step positive density condition and are therefore ergodic. 
Figures~\ref{fig:mnist-drift-plots} and~\ref{fig:imgnet-drift-plots} present 
the local and cumulative drift plots for the corresponding experiments, with 
the number of generations scaled to a maximum of 1. We begin by discussing the 
MNIST-based experiments.

\noindent \textbf{MNIST dataset:} Fig.~\ref{fig:mnist-drift-plots} tracks the local-drift 
indicator $\operatorname{FID}_{n,n-1}$ (blue curve) and 
the cumulative-drift indicator $\operatorname{FID}_{n,0}$ (orange curve) 
for the label-guided retrained, unconditional retrained, and 
latent-feedback GMCs. 

In the label-guided retrained GMC (100 steps, Fig.~\ref{fig:mnist-drift-plots} top-left-panel), the first $\sim$40 iterations 
show an active transient phase (sharp local drift decrease and cumulative drift increase). 
Between steps 40-90, the system enters a slow transient phase, 
and by $\sim$90 iterations both curves plateau, 
indicating empirical stationarity: the local drift plateaus, and the long-range distance ceases to 
grow. Here, the label-guidance signal reinforces this effect by anchoring each 
class manifold and preventing further cumulative drift, effectively keeping the 
distribution centered around class-conditional attractors.

For the \emph{unconditional retrained} GMC with $100$ iteration steps 
(Fig.~\ref{fig:mnist-drift-plots}, top-right panel), the local-drift curve (blue) 
drops sharply during the first 30 generations, after which the rate of decline 
slows to a small value. This indicates that each new generation still differs 
appreciably from its predecessor, with occasional spikes likely corresponding 
to mode-switching events. In contrast, the cumulative-drift curve (orange) 
rises almost linearly throughout the entire 100-step horizon, showing that the 
chain continues to move away from the original data manifold rather than 
settling into a stationary distribution. Taken together, a persistently high 
and fluctuating local drift, coupled with an unbounded cumulative drift, 
signals ongoing evolution.

In the \emph{latent-feedback} GMC with $250$ iteration steps 
(bottom panel of Fig.~\ref{fig:mnist-drift-plots}), the local-drift curve 
plunges during the first 20 iterations and then stabilizes at a small 
constant value. Meanwhile, the cumulative-drift curve rises steadily 
as the chain diverges from its initial distribution. During the first 100 
iterations, the chain remains in an active transient phase, followed by a slow 
transient phase for the next 100 iterations. Around the 200$^{th}$ 
iteration, the cumulative drift begins to taper off and eventually plateaus. 
This joint plateau pattern indicates empirical stationarity of the GMC.

\begin{figure*}[t]
    \centering
    \includegraphics[width=1\textwidth]{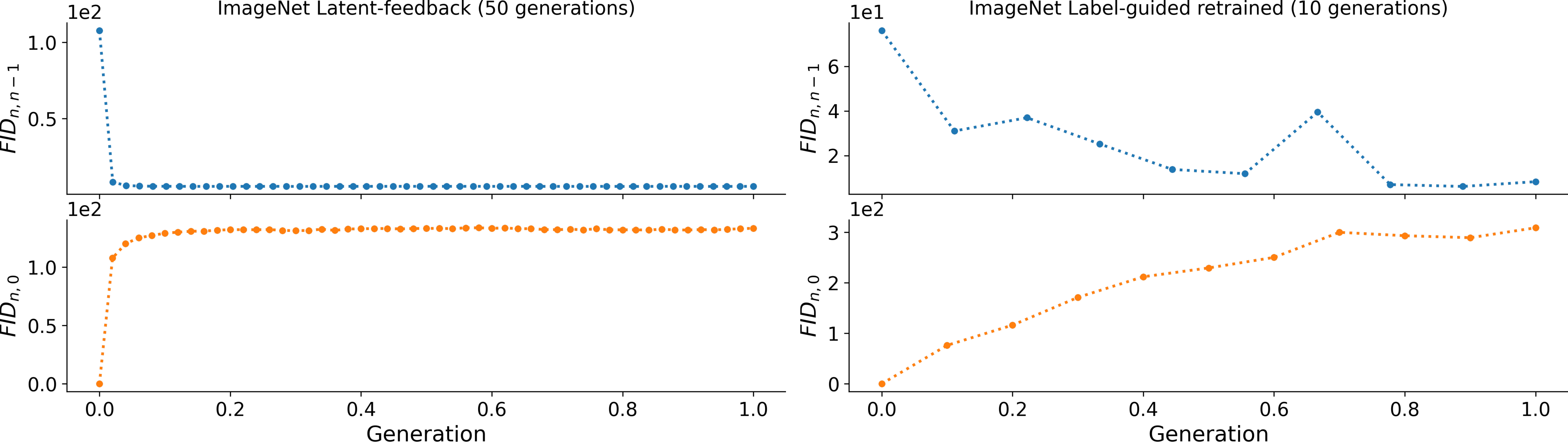}
    \caption{\textbf{Local and cumulative drifts in generational Markov chains on ImageNet-5.} 
    Each plot displays the local drift ($FID_{n,n-1}$, top sub-panel) 
    and the cumulative drift ($FID_{n,0}$, bottom sub-panel) across generations. 
    Results are shown for the latent-feedback diffusion model (left) 
    and the label-guided retrained diffusion model (right). 
    Both chains stabilize quickly, highlighting how the complexity of 
    the dataset can affect the behavior of the chain.}
    \label{fig:imgnet-drift-plots}
\end{figure*}

\noindent \textbf{ImageNet-5 dataset:} 
We experiment with two GMCs\textemdash latent-feedback and label-guided retrained\textemdash 
on the ImageNet-5 dataset at $128 \times 128 \times 3$ resolution. The 
diffusion-based framework provides the theoretical foundation for one-step 
positive density in both GMCs.

For the \emph{latent-feedback} GMC with $50$ iteration steps, we use InceptionV3 
features to provide latent feedback as conditional information in the diffusion 
process. Figure~\ref{fig:imgnet-drift-plots} (left) shows the drift plots for 
this GMC: the blue curve corresponds to local drift, and the orange curve to 
cumulative drift. Both curves change rapidly during the first few generations, 
followed by large plateaus that indicate empirical stationarity. 

For the \emph{label-guided retrained} GMC, we perform only $10$ iterations 
since the outputs begin to lose semantic meaning in nearly all classes by the 
5$^{\text{th}}$ generation. This is evident in the right panel of 
Figure~\ref{fig:imgnet-drift-plots}, where the drift plots show sharp changes 
during the first few iterations and little to no change during the last three. 
Here, the GMC reaches its stationary distribution very quickly. 

Across both GMCs, we observe that the chain converges to its stationary 
distribution within a few iterations. We posit that this behavior is 
influenced by the compressibility of the dataset: highly compressible images, 
such as MNIST digits, require only a small number of latent dimensions to 
faithfully represent the data, whereas ImageNet-5 images are far less 
compressible and demand significantly more dimensions to capture their semantic 
content (see global and local dimension plots in 
Figures~\ref{fig:quant-mnist} and~\ref{fig:quant-imgnet5}). This suggests that 
the number of iterations required to reach a stationary distribution is closely 
tied to the compressibility of the dataset.


\section{Neural Resonance: Expanded}\label{sec:neural-resonance}

The convergence of a generational Markov chain toward a stationary distribution mirrors, 
in the statistical domain, the acoustic phenomenon in Alvin Lucier's\textemdash \emph{I Am Sitting in a Room}. 
In Lucier’s piece, successive re-recordings of a spoken passage act 
as repeated applications of the room’s linear transfer operator; 
frequencies misaligned with the dominant eigenmodes decay exponentially, 
leaving only the resonant ``room chord.''  Likewise, each iterative step of a diffusion-based Markov kernel 
suppresses components orthogonal to the data manifold, 
so that after a finite mixing time only those feature-space directions 
whose associated eigenvalues cluster near one\textemdash the stationary law\textemdash persist. 
The exponential fall-off in ${FID}_{n,n-1}$ thus plays the 
same diagnostic role as the narrowing audio spectrum in Lucier’s experiment: 
both quantify the attenuation of sub-dominant modes and 
provide an empirical estimate of the spectral gap.  
Framed in this way, ergodicity is not merely a probabilistic guarantee 
but a form of ``\emph{neural resonance},'' whereby low probability modes are filtered 
through successive generative iterations until the invariant 
statistical structure of the data reverberates alone, 
like the resonant frequencies of a physical room.
Extending this fundamental principle beyond 
physical acoustics, we posit a direct conceptual analogy 
in generative neural networks, introducing the notion of \emph{neural resonance}. 
We define:
\begin{definition}[{\bf Neural resonance}]
    \label{def:neural-resonance}
    At the level of the entire data distribution, 
    neural resonance denotes the emergence of a low-dimensional invariant subspace that is preserved under 
    iterative generative updates.
\end{definition}
\noindent During convergence, the distribution variance projected onto directions orthogonal 
to this subspace decays exponentially, contracting its global effective dimensionality until it
plateaus at the dimension of the resonant space. Concurrently, 
variance within that emerging subspace\textemdash such as intra-class spread\textemdash 
can transiently grow, shrink, or remain stable, modifying the local structure of the clusters
without reversing the global collapse.
Not every iterative-feedback generative model is guaranteed to exhibit neural resonance;
it only shows when the two necessary conditions join forces:
\begin{itemize}
    \item \textbf{Directional contraction:} Each iterative-feedback step steadily attenuates
    a broad set of latent-space directions globally, shrinking their influence bit by bit 
    rather than leaving them unchanged.
    \item \textbf{Ergodicity:} No matter where the chain starts, its updates eventually 
    explore a range of behaviors and head toward a stationary distribution.
\end{itemize}
When both conditions hold, modestly damped directions quickly fade, 
while nearly neutral directions retain the probability mass. 
The distribution therefore collapses onto a thin, lower-dimensional ``resonant'' sheet 
that remains stable under further iterations. When  there is no directional contraction, the space never thins, 
and without ergodicity the generational chain can settle into disjoint attractors 
or drift without bound. Neural resonance pinpoints the regime in 
which global mixing and directional contraction cooperate to reveal\textemdash 
and continually reinforce\textemdash a persistent low-dimensional core, 
independent of how the chain began.

The resonance picture makes concrete, testable predictions about 
how \emph{global} and \emph{local} geometry evolve while the chain mixes.  
Globally, the participation ratio $\mathrm{PR}_G(g)$ should decline monotonically 
until it plateaus at the dimension of the resonant subspace, 
signaling the collapse of variance in all orthogonal directions.  
Locally, however, the within–class spread $\sigma_{\mathrm{intra}}(g)$
and the neighborhood embedding dimension $m_{\mathrm{LB}}(g)$ can transiently increase, 
decrease, or remain stable as residual noise and semantic detail 
are funneled into the shrinking subspace.  
By tracking the joint trajectory $\bigl(\sigma_{\mathrm{intra}}, m_{\mathrm{LB}}, \mathrm{PR}_G\bigr)$
we can understand the dynamics of the latent manifold, 
\emph{how} intra-class spread reorganizes en route, 
and \emph{which} latent-manifold regime\textemdash semantic expansion or contraction\textemdash dominates.  
The following subsection formalizes the eight qualitatively distinct sign-patterns that can arise.

\begin{table}[t]
  \centering
  \caption{\textbf{Latent‑manifold dynamics under neural resonance.}
    Each row lists a distinct combination of directional change in the local spread ($\sigma_{intra}$), 
    the local curvature measure ($m_{LB}$), and the global participation ratio ($PR_G$). 
    Patterns are grouped into two regimes\textemdash semantic expansion ($\sigma_{intra}$ increasing) 
    and semantic contraction ($\sigma_{intra}$ decreasing)\textemdash and each is labeled 
    according to its geometric character (e.g., coherent expansion, wrinkled contraction). 
    An animation illustrating these patterns is provided in the Supplementary Video.}
  \begin{tabular}{llcccl}
    \toprule
    \multicolumn{2}{c}{\textbf{Manifold Behavior}} 
      & \multicolumn{2}{c}{\textbf{Local Metrics}} & \textbf{Global Metric} & \textbf{Dimensional Pattern}\\
    \cmidrule(lr){1-2} \cmidrule(lr){3-4} \cmidrule(lr){5-5}
    Local & Global & $\sigma_{intra}$ & $m_{LB}$ & $PR_G$ & \\
    \midrule
    \addlinespace[3pt]
    \multicolumn{6}{l}{\textbf{Semantic Expansion Regime}} \\
    \addlinespace[3pt]
    Expansion       & Expansion    & $\uparrow$ & $\uparrow$ & $\uparrow$ & Coherent Expansion (CE)\\[2pt]
    Expansion       & Contraction  & $\uparrow$ & $\uparrow$ & $\downarrow$ & Wrinkled Expansion (WE)\\[2pt]
    Expansion     & Expansion    & $\uparrow$ & $\downarrow$ & $\uparrow$ & Anisotropic Expansion (AE)\\[2pt]
    Expansion     & Contraction  & $\uparrow$ & $\downarrow$ & $\downarrow$ & Oblate Expansion (OE)\\[2pt]
    \midrule
    \addlinespace[2pt]
    \multicolumn{6}{l}{\textbf{Semantic Contraction Regime}} \\
    \addlinespace[2pt]
    Contraction    & Contraction  & $\downarrow$ & $\downarrow$ & $\downarrow$ & Coherent Contraction (CC)\\[2pt]
    Contraction   & Expansion    & $\downarrow$ & $\downarrow$ & $\uparrow$ & Anisotropic Contraction (AC)\\[2pt]
    Contraction     & Contraction  & $\downarrow$ & $\uparrow$   & $\downarrow$ & Wrinkled Contraction (WC)\\[2pt]
    Contraction     & Expansion    & $\downarrow$ & $\uparrow$   & $\uparrow$ & Oblate Contraction (OC)\\[2pt]
    \bottomrule
  \end{tabular}
  \label{table:ideal-manifold-behaviors-appendix}
\end{table}
\subsection{Latent Manifold Dynamics: Expanded}\label{subsec:latent-dynamics}

We categorize the semantic phenomena according to their local impact on 
the latent manifold 
into two overarching regimes based on the increase and decrease in $\sigma_{intra}$: 
the \emph{semantic‐expansion} ($\sigma_{\mathrm{intra}}\uparrow$) 
and the \emph{semantic‐contraction} ($\sigma_{\mathrm{intra}}\downarrow$) regimes. 
Table \ref{table:ideal-manifold-behaviors-appendix} systematically catalogs every 
qualitatively distinct combination of local and global manifold behaviors.
An animation illustrating the dimensional patterns is included in the supplementary video.
Within the semantic expansion regime, the four rows describe the behavior of
$m_{LB}$ and $PR_G$ with increase in $\sigma_{intra}$. 
Similarly, the semantic contraction regime enumerates behaviors of 
$m_{LB}$ and $PR_G$ 
for a decreasing $\sigma_{intra}$. 
Collectively, these eight sign‐patterns ensure 
that any empirical latent‐manifold trajectory can be classified 
and interpreted within this unified framework.
We briefly describe the dimensional behavior of latent manifolds below:
\begin{itemize}
  \item \textbf{Coherent Expansion} ($\mathbf{\sigma_{intra}}\uparrow,\mathbf{m_{LB}}\uparrow,\mathbf{PR_G}\uparrow$): increase in 
  $\sigma_{intra}$ activates new axes of local variations, as evidenced by increase in $m_{LB}$. 
  These new axes remain coherent across classes and align 
  with the principal directions carrying significant global variance, increasing $PR_G$. 
  Geometrically, the manifold simultaneously puffs out and unfolds into previously unoccupied subspaces, 
  analogous to a balloon inflating so that new portions of its surface extend into directions that were previously unoccupied.
  \item \textbf{Wrinkled Expansion} ($\mathbf{\sigma_{intra}}\uparrow,\mathbf{m_{LB}}\uparrow,\mathbf{PR_G}\downarrow$): increase in 
  $\sigma_{intra}$ forms small‐scale `wrinkles' or nonlinear creases within each cluster, 
  increasing $m_{LB}$. However, these wrinkles do not align with the principal 
  direction carrying significant global variance, 
  effectively decreasing $PR_G$. 
  Geometrically, the manifold fans out locally in new directions as wrinkles, creases or folds, 
  but those directions cancel or average out globally, much like a sheet of fabric crumpling into 
  many small folds that do not produce a consistent large-scale orientation.
  \item \textbf{Anisotropic Expansion} ($\mathbf{\sigma_{intra}}\uparrow,\mathbf{m_{LB}}\downarrow,\mathbf{PR_G}\uparrow$): increase in 
  $\sigma_{intra}$ stretches each cluster anisotropically along different major axes, 
  causing a decrease in $m_{LB}$. Nevertheless, this anisotropic stretching along different axes for different clusters 
  activates additional global principal directions of significant variance, increasing $PR_G$. 
  Thus, each cluster `grows out' along its own axis\textemdash collapsing local 
  dimensions while broadening the global footprint, much like lumps of play-dough being 
  stretched in different directions, each elongating along its own axis while together 
  occupying a wider region.
  \item \textbf{Oblate Expansion} ($\mathbf{\sigma_{intra}}\uparrow,\mathbf{m_{LB}}\downarrow,\mathbf{PR_G}\downarrow$): increase in 
  $\sigma_{intra}$ expands the clusters along one or a few dominant directions, causing local clusters to oblate 
  and collapse onto those axes, decreasing $m_{LB}$. 
  Globally, the same (old) principal axes capture this increase in local variance, decreasing $PR_G$. 
  The manifold oblates outward, similar to a balloon compressed between two plates, 
  expanding sideways along a few axes while collapsing along others.
  \item \textbf{Coherent Contraction} ($\mathbf{\sigma_{intra}}\downarrow,\mathbf{m_{LB}}\downarrow,\mathbf{PR_G}\downarrow$): This is the
  reverse of the \emph{coherent expansion} process. The decrease in $\sigma_{intra}$ removes new 
  axes of local variations, decreasing $m_{LB}$. These disappearing local axes remain coherent across 
  classes and align with the principal directions carrying significant 
  global variance, decreasing $PR_G$.
  \item \textbf{Anisotropic Contraction} ($\mathbf{\sigma_{intra}}\downarrow,\mathbf{m_{LB}}\downarrow,\mathbf{PR_G}\uparrow$): The dimensional 
  pattern of anisotropic contraction is 
  reverse of \emph{wrinkled expansion} process. The decrease in $\sigma_{intra}$ removes the 
  `wrinkles' or 'nonlinear creases' from the clusters converting into anisotropic cluster with different major axes.
  This local contraction decreases $m_{LB}$, 
  but different clusters contract along distinct axes, activating new global principal directions of significant variance, 
  increasing $PR_G$. In other words, each class shrinks along 
  its own preferred axis anisotropically, causing the ensemble to diversify across more 
  global modes despite overall contraction.
  \item \textbf{Oblate Contraction} ($\mathbf{\sigma_{intra}}\downarrow,\mathbf{m_{LB}}\uparrow,\mathbf{PR_G}\uparrow$): The dimensional 
  pattern of oblate contraction is the 
  reverse of \emph{oblate expansion} process. The decrease in $\sigma_{intra}$ contracts the 
  clusters along one or a few dominant directions, causing the local clusters to inflate along 
  those axes, increasing $m_{LB}$. Globally, new principal directions capture this decrease in 
  local variance, increasing $PR_G$.
  \item \textbf{Wrinkled Contraction} ($\mathbf{\sigma_{intra}}\downarrow,\mathbf{m_{LB}}\uparrow,\mathbf{PR_G}\downarrow$): The dimensional 
  pattern of the wrinkled contraction is the
  reverse of \emph{anisotropic expansion} process. The decrease in $\sigma_{intra}$ de-stretches each 
  clusters creating wrinkles or nonlinear creases, increasing $m_{LB}$. Nevertheless, 
  these de-stretching removes some of the previously required principal direction of variance, causing 
  $PR_G$ to decrease. 
\end{itemize}

\begin{figure*}[t]
    \centering
    \includegraphics[width=1\textwidth]{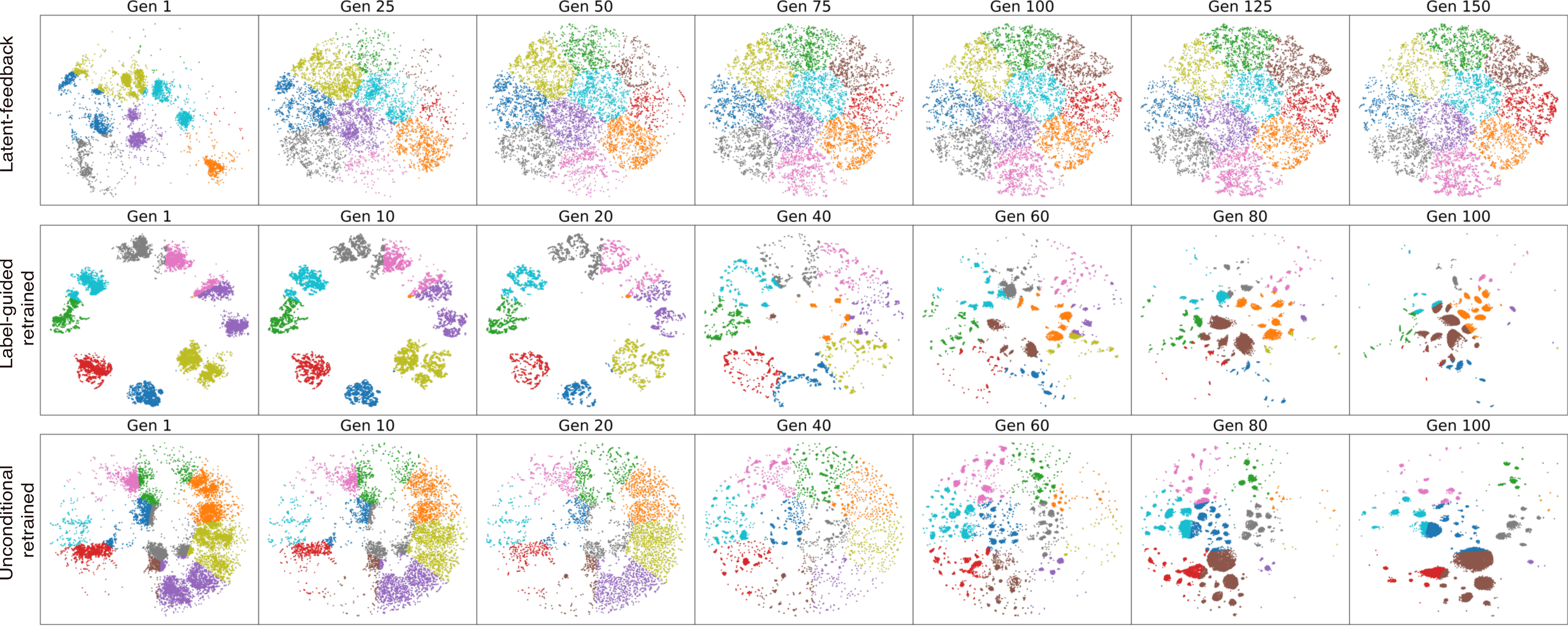}
    \caption{\textbf{t-SNE visualization of generational Markov chains on MNIST.} 
    Each plot shows the two-dimensional t-SNE embedding of samples drawn from a generational Markov chain, grouped into ten clusters based on latent-space proximity rather than digit class. 
    For the label-guided and unconditional retrained models, 
    the first generation is trained on the MNIST training set 
    and each subsequent generation uses 50,000 images. 
    The latent-feedback chain starts from the MNIST test set, with 10,000 images per generation.}
    \label{fig:tsnes}
\end{figure*}

\subsection{Dimensional Pattern of GMCs}

\noindent To characterize the latent manifold dynamics of GMCs, we begin with low-dimensional 
projections (t-SNE) of their latent manifolds. These visualizations, see Fig.~\ref{fig:tsnes}, reveal that different GMCs 
follow distinct evolutionary patterns, reflected in the organization 
and diffusion of clusters over time. Building on these observations, 
we then turn to a quantitative analysis of global (participation ratio) 
and local (Levina-Bickel intrinsic dimensionality) manifold behaviors. 
By tracking intra-class spread, intrinsic dimensionality, and participation ratio, 
we assign characteristic dimensional patterns discussed in Table~\ref{table:ideal-manifold-behaviors-appendix} 
to different iterative phases of GMCs.

\subsubsection{Low-Dimensional Projections of GMCs}

To visualize the low-dimensional projections of GMCs in latent space, we compute
a single t-SNE embedding jointly optimized with normalize perplexity and learning-rate across 
generation, ensuring comparability of all samples at each generation. 
t-SNE is a nonlinear dimensionality-reduction technique that maps 
high-dimensional data into two or three dimensions by converting pairwise similarities into 
joint probability distributions and then minimizing their Kullback-Leibler divergence~\cite{vandermaaten08a}. 
Fig. ~\ref{fig:tsnes} presents 2D t-SNE manifold projections for the three ergodic GMCs\textemdash 
latent-feedback, label-guided retrained, and unconditional-retrained\textemdash on the MNIST dataset. 
Each plot is computed from the same latent-manifold activations used for other metric estimations. 
They are jointly estimated with identical perplexity and learning-rate settings. 
Even in two dimensions there is a clear indication of different GMCs following different 
dimensional patterns. 

The \emph{latent-feedback} GMC (top) begins with compact clusters that quickly diffuse into previously 
unoccupied regions, consistent with the early rise in $\sigma_{intra}$. The local expansion affects the 
local and global manifold dimension differently. The subspace stabilizes or evolves very slowly towards the end. 
The \emph{label-guided retrained} GMC (middle) begins with sufficient cluster sizes for each class, but 
the clusters shrink without mixing with other classes (class-label acting as `anchors', no sample from cluster boundaries), 
indicating a significant decrease in feature variance $\sigma_{intra}$ around the middle generations. 
In later generations, $\sigma_{intra}$ plateaus, indicating no significant change in distribution. 
The \emph{unconditional-retrained} GMC (bottom) also starts with large clusters, 
but in the absence of class-label anchors, the clusters slowly 
shrink (decrease in $\sigma_{intra}$) 
and drift (with more samples from cluster boundaries appearing as incomplete digits) at the same time. 
As the GMC progresses, the distribution includes more incomplete digits, 
and digit 4 completely disappears from the sampled distribution 
around the $87^{th}$ generation. This continues until 
the end of our experiment at the $100^{th}$ generation.

Taken together, the t-SNE maps corroborate the quantitative curves: each GMC
follows a distinct trajectory toward its stationary distribution, and changes in local spread
($\sigma_{intra}$) reshape both global and local latent-manifold
geometry in different ways, consistent with the neural-resonance framework.

\begin{figure*}[t]
    \centering
    \includegraphics[width=1\textwidth]{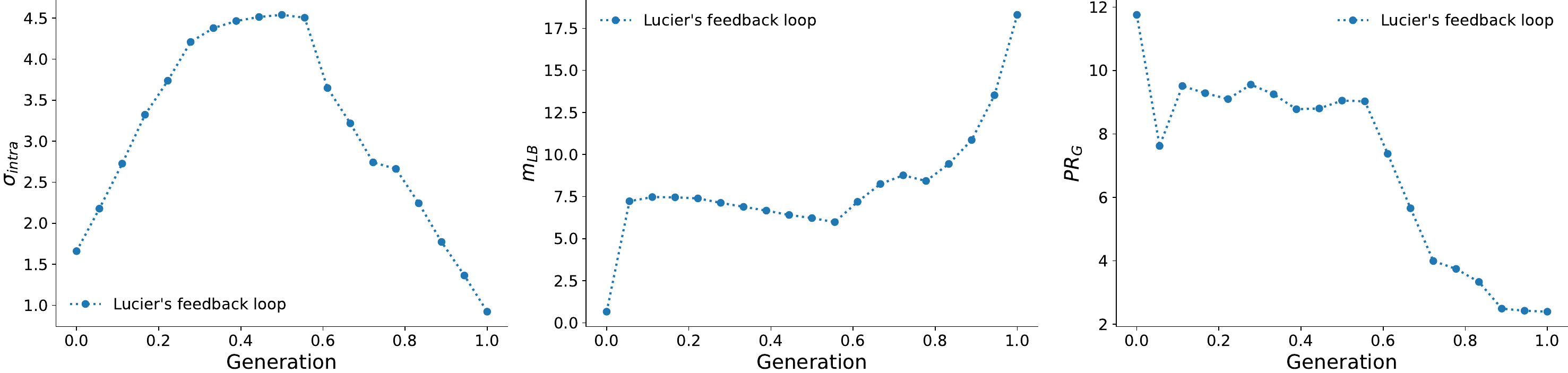}
    \caption{\textbf{Dimensional pattern of a functional analogue of Lucier’s feedback GMC.} 
    We track the dimensional pattern of the GMC using three metrics\textemdash $\sigma_{intra}$, 
    $m_{LB}$, and $PR_G$\textemdash computed in the Wav2Vec 2.0 latent space 
    for an audiobook of \emph{War and Peace} convolved with OpenAIR impulse-response signals. 
    In the early phase the latent manifold undergoes a wrinkled expansion, 
    followed by an oblate expansion in the middle phase, 
    and a wrinkled contraction in the final phase, nearly reducing $PR_G$ to zero
    and driving $m_{LB}$ to very large values. The chain does not exhibit neural resonance.}
    \label{fig:quant-lucier}
\end{figure*}
\subsubsection{Global and Local Dimensions of GMCs}

We track the dynamics of the latent-manifold of GMCs using intra-class spread $\sigma_{intra}$, 
intrinsic dimensionality $m_{LB}$, and participation ratio $PR_G$ (discussed in Sec.~\ref{sec:dataset-and-metrics}). 
We discuss the characteristics of the GMCs and interpret the manifold dynamics using 
the dimensional patterns described in Table~\ref{table:ideal-manifold-behaviors-appendix} \& 
Sec.~\ref{subsec:latent-dynamics}. To systematically analyze the dimensional patterns, we divide 
the whole iterative process into three phases\textemdash the \emph{early} phase, the \emph{middle} phase, and 
the \emph{late} phase.

\paragraph{Lucier's Feedback Loop}
We replicate a functional analogue of Lucier's experiment on a recording of a reading of \textit{War \& Peace} using 10 different
impulse responses from the OpenAIR dataset. We consider each impulse response as a class. This generation 
Markov chain does not have one-step positive density (Lemma~\ref{lemma:one-step-density}) and is not an ergodic GMC. 

For latent-manifold tracking, we segmented each recording into non-overlapping partitions of 20 seconds and encoded every 
window with Wav2Vec 2.0 embeddings~\cite{baevski2020wav2vec}. 
Fig.~\ref{fig:quant-lucier} charts the latent-manifold dynamics across iterations. 
The early phase shows \emph{wrinkled-expansion} where 
both $\sigma_{intra}$ and $m_{LB}$ increase but $PR_G$ decreases. The middle phase shows 
dynamics close to \emph{oblate-expansion} behavior marked by slow increase of $\sigma_{intra}$, 
and a slow decrease of $m_{LB}$ and $PR_G$. 
The late phase shows \emph{wrinkled-contraction} marked by decrease in $\sigma_{intra}$ and $PR_G$ while $m_{LB}$ increases. 
Despite the sustained directional contraction in the neural or embedding space, 
the functional analogue of Lucier's experiment does not show neural resonance\textemdash 
the original physical experiment does exhibit acoustic resonance\textemdash 
because the corresponding GMC is non-ergodic.

\begin{figure*}[t]
    \centering
    \includegraphics[width=1\textwidth]{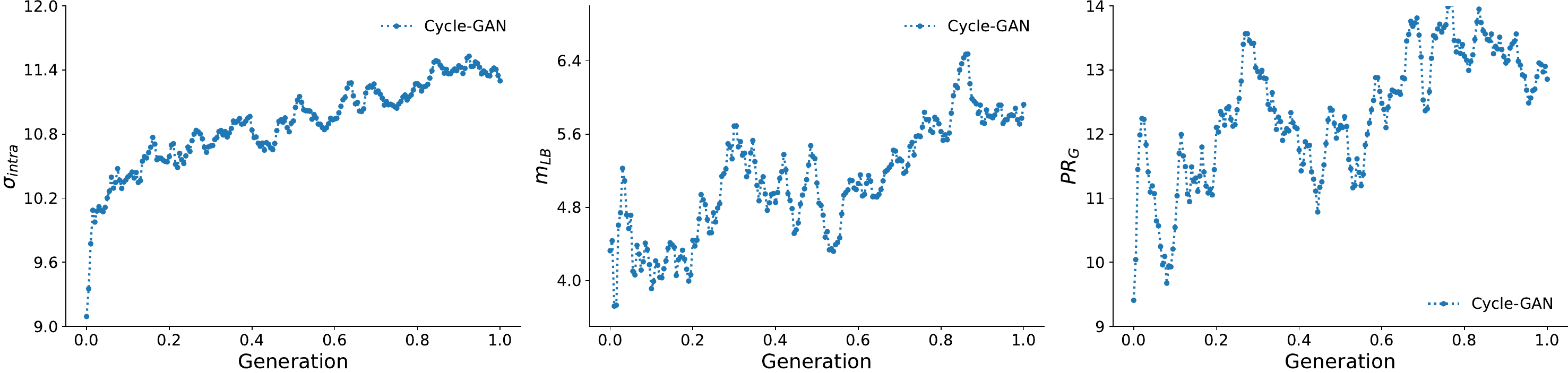}
    \caption{\textbf{Dimensional pattern of CycleGAN GMC} From left to right, the plot
    tracks intra-class spread, local and global dimensions on Horse-Zebra dataset for 200 iterations.}
    \label{fig:quant-cyclegan}
\end{figure*}
\paragraph{Cyclic Image Translation}\label{subsec:cycle-gan}
We apply CycleGAN to the Horse-Zebra corpus, yielding a fully deterministic 
transformation operator. Because the update map lacks a one-step positive 
density (Lemma~\ref{lemma:one-step-density}), the associated generative 
Markov chain is non-ergodic: it does not mix toward a stationary 
distribution but instead settles into a limit cycle governed by 
the reconstruction loss and the two image domains.

Latent-manifold dynamics are tracked with features from the penultimate 
layer of a pretrained Inception-V3 encoder. Fig.~\ref{fig:quant-cyclegan} reports the 
three metrics over 200 iterations. Generation-to-generation behavior is irregular\textemdash 
each generator repeatedly attempts to undo its counterpart’s output\textemdash 
but the aggregate trend resembles a \emph{coherent-expansion}: the within-class 
spread $\sigma_{intra}$, the local dimensionality $m_{LB}$, and the 
global participation ratio $PR_G$ all increase. Because neither global 
contraction nor mixing is observed, this CycleGAN chain never enters the neural-resonance regime.

\begin{figure*}[t]
    \centering
    \includegraphics[width=1\textwidth]{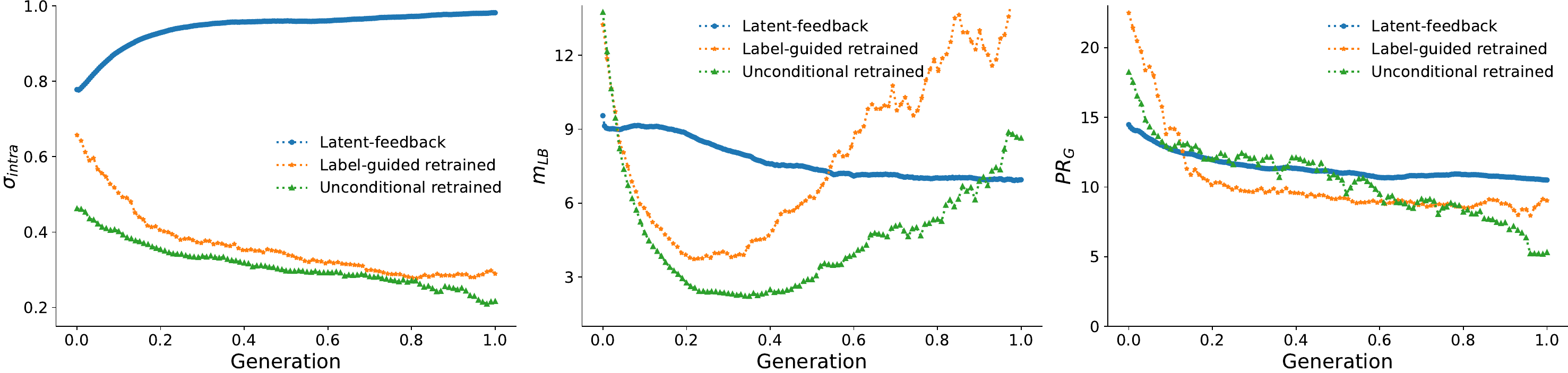}
    \caption{\textbf{Dimensional pattern of GMCs on MNIST dataset.} 
    From left to right, the plot tracks intra-class spread, local and 
    the global dimensions of latent-feedback (blue, 200 iterations), 
    label-guided retrained (orange, 100 iterations) 
    and unconditional retrained (green, 100 iterations) GMCs.}
    \label{fig:quant-mnist}
\end{figure*}

\paragraph{Latent-Feedback Diffusion Models}\label{subsec:latent-feedback}
We run latent-feedback experiments with diffusion models on the MNIST and ImageNet-5 datasets.
Because a diffusion step admits a one-step positive density (Lemma \ref{lemma:one-step-density}) 
and the resulting feedback chain also meets the criteria of local self-transition, 
petite sets, and Harris recurrence (Lemmas \ref{lemma:local-self-transition}–\ref{lemma:harris-recurrece}), 
the associated generative Markov chain is ergodic.  Given enough feedback iterations, 
it should therefore mix toward a stationary distribution.

For the MNIST experiment we monitor latent-manifold dynamics over 250 feedback 
iterations using features from a pretrained classifier. Fig. \ref{fig:quant-mnist} (blue curve) 
shows the first 100 iterations, after which the behavior remains qualitatively unchanged. 
Throughout, the latent-feedback chain exhibits an \emph{oblate-expansion} pattern: 
the within-class spread $\sigma_{intra}$ rises early and then plateaus, 
while both the local dimensionality $m_{LB}$ and the global participation ratio $PR_G$ 
contract steadily before stabilizing at a fixed value. Because the process satisfies 
ergodicity and displays strong directional contraction, it meets the two criteria for neural resonance.

For the ImageNet-5 latent-feedback chain we monitored manifold evolution using features from 
the penultimate layer of a pretrained Inception-V3 encoder. Fig. \ref{fig:quant-imgnet5} (blue curve) 
depicts the first 50 feedback iterations. 
In the early stages both the within-class spread $\sigma_{intra}$ and 
the local intrinsic dimensionality $m_{LB}$ rise before leveling off, 
while the global participation ratio $PR_G$ collapses rapidly and then plateaus. 
This combination\textemdash local variance inflating inside a concurrently shrinking global subspace\textemdash 
matches the \emph{wrinkled-expansion} regime. Because the chain is ergodic and 
displays strong directional contraction, it satisfies the two conditions required for neural resonance.

\begin{figure*}[t]
    \centering
    \includegraphics[width=1\textwidth]{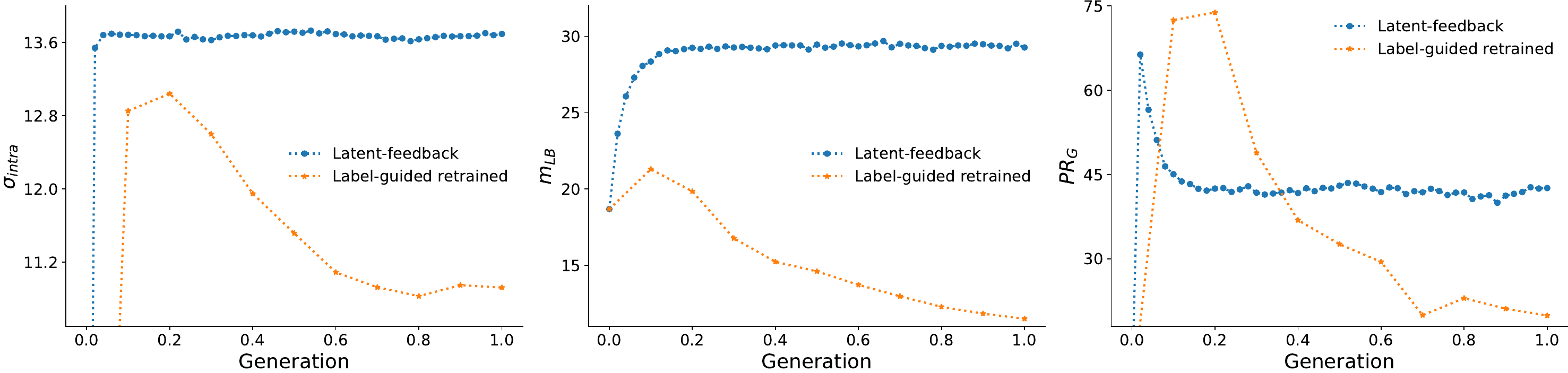}
    \caption{\textbf{Dimensional pattern of GMCs on ImageNet-5 dataset.}
    From left to right, the plot tracks intra-class spread, intrinsic dimensionality and 
    global dimensions of the latent-feedback (blue, 50 iterations) 
    and label-guided retrained (orange, 11 iterations) diffusion models.}
    \label{fig:quant-imgnet5}
\end{figure*}
\paragraph{Label-Conditional Retrained Diffusion Models}\label{subsec:cond-retrained}
We run label-conditional retrained model experiments with diffusion models on the MNIST 
and ImageNet-5 datasets. These GMCs are also ergodic satisfying 
Lemmas \ref{lemma:one-step-density}-\ref{lemma:harris-recurrece}.
Given enough feedback iterations, it should therefore mix towards a stationary distribution. 

In the MNIST study we track the generative Markov chain for 100 feedback iterations; 
its latent-manifold trajectory is plotted as the orange curve in Fig. \ref{fig:quant-mnist}.
During the initial iterations the chain undergoes a pronounced \emph{coherent-contraction}: 
all three metrics\textemdash within-class spread $\sigma_{intra}$, local dimension $m_{LB}$, 
and global participation ratio $PR_G$\textemdash fall sharply.
From the middle phase onward, $m_{LB}$ begins a gradual rebound 
while $\sigma_{intra}$ and $PR_G$ stabilize or continue to decline, 
marking a transition to \emph{wrinkled-contraction}.
Because the chain is ergodic and exhibits sustained directional damping, 
it fulfills the two criteria required for neural resonance.

In the ImageNet-5 experiment we track the chain for only ten feedback iterations; 
the orange curve in Fig. \ref{fig:quant-imgnet5} reveals an aggressive, 
monotonic decline in all three manifold metrics. 
By iteration 10 the outputs have deteriorated to near-random noise 
with only faint object silhouettes remaining. The chain nevertheless remains ergodic 
and exhibits rapid directional contraction, thereby satisfying the prerequisites for neural resonance.

\paragraph{Unconditional Retrained Diffusion Model}\label{subsec:uncond-retrained}
We evaluate the unconditional retrained diffusion model solely on the MNIST dataset. 
Like the latent-feedback and label-guided variants, this chain possesses 
a positive one-step density and satisfies the petite-set and Harris-recurrence conditions, 
making it ergodic; with sufficient iterations it should therefore 
converge to a stationary distribution.

The chain is run for 100 feedback steps, with each batch re-classified by a pretrained 
digit classifier before manifold analysis. The green curve in Fig. \ref{fig:quant-mnist} summarizes 
its trajectory: the within-class spread $\sigma_{intra}$ and the global participation ratio $PR_G$ 
fall almost linearly, while the intrinsic dimensionality $m_{LB}$ drops sharply 
at first and then rebounds slightly. Although the process meets both prerequisites for 
neural resonance\textemdash ergodicity and sustained directional contraction\textemdash the characteristic resonant 
plateau (see Fig.\ref{fig:mnist-drift-plots} top right panel) has not 
yet appeared. This is also evident in the $PR_G$ plot, where 
the global dimensionality continues to contract without showing signs of stabilization
within the first 100 iterations.

\section{Discussion}

Our study offers a unified theoretical and empirical perspective on 
how generational Markov chains (GMCs)\textemdash iterative models that recycle their own outputs\textemdash evolve 
in latent space. By treating examples as diverse as the functional analogue of Lucier’s acoustic feedback loop, CycleGAN image translation, 
latent-feedback diffusion sampling, and both unconditional and label-guided retraining models, 
we derive general conditions under which a GMC is \emph{ergodic}. To assess its empirical stationarity, 
we introduce two complementary FID traces: $\operatorname{FID}_{n,n-1}$, which 
serves as a \emph{local-drift} indicator, and $\operatorname{FID}_{n,0}$, 
which acts as a \emph{cumulative-drift} indicator. When both 
metrics change rapidly the chain is in an \emph{active-transient} phase; 
when one curve evolves slowly while the other has flattened, the 
system is in a \emph{slow-transient} phase; and when 
both traces plateau simultaneously, the GMC has reached empirical stationarity.

Drawing inspiration from Alvin Lucier’s original experiment, 
which exhibits acoustic resonance (but our functional analogue lacks ergodicity, hence no neural resonance), 
and examining our formulation 
of GMCs across diverse settings, we identified an analogous 
phenomenon in latent feature space, which we term \emph{neural resonance}.
Two conditions are necessary for its emergence: (i) \emph{ergodicity}, which 
guarantees convergence to a stationary distribution, 
and (ii) \emph{directional contraction}, which drives variance 
in non-resonant directions toward zero. When both criteria are met, 
the GMC collapses onto a low-dimensional, update-invariant subspace; 
when either is violated, the chain either cycles within a finite set of attractors 
(as in CycleGAN image translation) or drifts without bound. We link 
this theory to three empirical manifold metrics\textemdash $\sigma_{\text{intra}}$, $m_{LB}$, and $\mathrm{PR}_{G}$\textemdash 
whose joint trajectories reveal eight characteristic dimensional patterns in latent space. Together, 
these measures form a compact yet powerful toolkit for diagnosing the dynamics of iterative generative systems.

Several caveats temper these findings. All quantitative metrics were extracted 
in fixed feature spaces (Inception-V3 for images and Wav2Vec 2.0 for audio); 
the absolute plateau heights of the drift curves therefore depend on the specific 
embedding. The benchmark suites 
employed here, MNIST, ImageNet-5, and OpenAIR, are modest in scale compared 
with contemporary foundation models, so the interplay between neural resonance 
and transformer-sized architectures remains an open research question.

We highlight two directions for future exploration.
\emph{(i) Mitigation.} Designing noise 
schedules or regularizers that either delay or suppress generational collapse 
without sacrificing sample fidelity. \emph{(ii) Cross-modal generalization.} Extending the 
resonance framework to text-only and vision-language models, where 
directional contraction must be defined in semantic rather than pixel space.

In summary, the combination of Markov-chain theory and feature-space geometry 
presented in this work provides a principled lens through which to analyze\textemdash and 
ultimately control iteratively trained generative models. The empirical signatures 
introduced here should enable practitioners to detect generational collapse 
at an early stage and intervene before complete collapse of diversity. As generative AI systems 
assume ever more autonomous roles, such safeguards will be essential for 
preserving fidelity, variety, and, ultimately, trustworthiness in their outputs.


\bibliographystyle{unsrt}
\bibliography{paper}

\end{document}